\documentclass{article}

\usepackage{microtype}
\usepackage{graphicx}
\usepackage{subfigure}
\usepackage{booktabs} %

\usepackage{soul}%

\usepackage{hyperref}
\usepackage{capt-of}

\usepackage[accepted]{ICML_stylefiles/icml2024}
\usepackage{ICML_stylefiles/algorithm}
\usepackage{ICML_stylefiles/algorithmic}
\usepackage{ICML_stylefiles/fancyhdr}

\usepackage{amsmath}
\usepackage{amssymb}
\usepackage{mathtools}
\usepackage{amsthm}
\usepackage[capitalize,noabbrev]{cleveref}

\theoremstyle{plain}
\newtheorem{theorem}{Theorem}[section]

\theoremstyle{definition}

\theoremstyle{remark}

\usepackage{comment}

\usepackage[textsize=tiny]{todonotes}

\usepackage{graphicx} %
\usepackage{xcolor}
\usepackage{float}
\usepackage{tikz}
\usetikzlibrary{trees}

\newcommand{\iid}{\overset{iid}{\sim}}
\newcommand{\E}{\mathbb{E}}
\newcommand{\V}{\mathbb{V}}
\DeclareMathOperator*{\argmax}{arg\,max}

\newcommand{\indep}{\perp \!\!\! \perp}

\newcommand{\simiid}{\stackrel{iid}{\sim}}

\icmltitlerunning{Missing Data Multiple Imputation for Tabular Q-Learning in Online RL
}

\begin{document}

\twocolumn[ %
\icmltitle{Missing Data Multiple Imputation for Tabular Q-Learning in Online RL
}

\icmlsetsymbol{equal}{*}

\begin{icmlauthorlist}
\icmlauthor{Kyla Chasalow}{equal,yyy}
\icmlauthor{Skyler Wu}{equal,zzz}
\icmlauthor{Susan Murphy}{yyy,xxx}
\end{icmlauthorlist}

\icmlaffiliation{yyy}{Department of Statistics, Harvard University, Cambridge, United States}
\icmlaffiliation{zzz}{Department of Statistics, Stanford University, Stanford, United States}
\icmlaffiliation{xxx}{Department of Computer Science, Harvard University, Cambridge, United States}

\icmlcorrespondingauthor{Kyla Chasalow}{kyla\_chasalow@g.harvard.edu}
\icmlcorrespondingauthor{Skyler Wu}{skylerw@stanford.edu}

\icmlkeywords{Reinforcement Learning, ICML}

\vskip 0.3in
]

\printAffiliationsAndNotice{\icmlEqualContribution} %

\begin{abstract}
Missing data in online reinforcement learning (RL) poses  challenges compared to missing data in standard tabular data or in offline policy learning. The need to impute and act at each time step means that imputation cannot be put off until enough data exist to produce stable imputation models. It also means future data collection and learning depend on previous imputations.  This paper proposes fully online imputation ensembles. We find that 
maintaining multiple imputation pathways may help balance the need to capture uncertainty under missingness and the need for efficiency in online settings. We consider multiple approaches for incorporating these pathways into learning and action selection. Using a Grid World experiment with various types of missingness, we provide preliminary evidence that multiple imputation pathways may be a useful framework for constructing simple and efficient online missing data RL methods. 
\end{abstract}

\section{Introduction} %

Missing data is a reality of working with real-world data. In online Reinforcement Learning (RL), where the goal is to learn how to act to maximize a reward while interacting with a dynamic environment in real time, there is an added complexity that the way we deal with missing data at earlier times affects subsequent data collection and learning. Unlike in offline-RL or traditional tabular data, it is not enough to design a single imputation model. The agent confronts missingness repeatedly.%

Compared to offline or batch-RL, in fully online RL it is especially critical that decisions are made in a computationally efficient way. 
Methods also need to be autonomous and not require a human-in-the-loop. Particularly when a learning algorithm is part of an intervention, as in mobile health studies, it is important that the algorithm is pre-specified and not altered during the study \cite{Ghosh2024}. Such algorithms must therefore minimize the risk of  bad behavior. 

Outside of RL, multiple imputation (MI) is an established framework for dealing with missing data. Imputations are drawn  independently from an imputation model to create a number of completed datasets, allowing analysts to run standard analyses multiple times and combine the results while quantifying uncertainty due to missingness \cite{Schafer_1999}. Transferring MI and its properties to online RL is not immediate given the complex dependencies that can arise if past imputations and imputation-based decisions inform future ones. However, the success of ensemble-based methods in RL \citep{WieringHado2008, fausser2015neural, Chen_etal_2021} %
suggests the possibility of creating multiple imputation pathways and combining them to decide how to act at each time step.

In this paper, we explore multiple imputation ensembles for online RL with missing state space data motivated by two intuitions: (1) 
probabilistic  imputations using  learned transitions may outperform simpler baselines if they allow agents to better leverage partial information (2) by better representing uncertainty, imputation ensembles may avoid issues of path dependency that could emerge from a single imputations under repeated missingness.  We focus for now on tabular Q-learning in order to uncover core dynamics and insights. We map out a number of design choices involved in creating these ensembles, including fractional updating, action-selection, and whether to treat past imputations as if they are real data for learning. Then we provide a grid world experiment that demonstrates their potential to out-perform simple baselines and single imputation under different missingness scenarios. We also compare these methods to encoding missingness as a state option, revealing a U-shaped performance curve as a function of missingness rate that we believe occurs because of how this approach effectively increases or reduces the state space dimension. %

After establishing background and notation in Section \ref{sec-prelim}, we introduce our methods in Section \ref{sec-methods}. Section \ref{sec-experiment} contains experiments on a custom grid world with missingness. Section \ref{sec-discussion} discusses and concludes.

\section{Background and Related Literature}\label{sec-prelim} 
\subsection{Reinforcement Learning}

We consider the finite state and action-space Markov Decision Process (MDP) in which the environment is a Markovian sequence $(S_t,A_t,R_{t+1},S_{t+1})$ for $t=0,1,2,...$ \cite{SuttonBarto2018}. At time $t$, the state $S_t$  takes values in $\mathcal{S}\subseteq \mathbb{R}^d$, and the agent selects action $A_t$ from an action space $\mathcal{A}$ with $|\mathcal{A}|<\infty$. Given a state-action pair $(S_t,A_t)=(s,a)$, the environment transitions to state $S_{t+1}$ according to  transition distribution $T(s'|s,a)$ with $p_0$ the distribution of initial state $S_0$. The agent also receives a reward $R_{t+1}\in \mathcal{R}$ for some bounded space $\mathcal{R}$ and where the mean reward is $r(s,a)$. %
The goal of RL is to learn a policy $\pi(a|s)$ which picks an optimal action given the current state despite environment dynamics $r(s,a)$ and $T(s'|s,a)$ being unknown. Because in an MDP, actions have delayed consequences, learning an optimal policy $\pi^*$ requires optimizing the expected discounted sum of future rewards  known as 
the value function $\V^\pi(s)=\E[G_t|S_t=s]$, where $G_t = \sum_{k=0}^{\infty}\gamma^k R_{t+k+1}$, with discount $\gamma \in [0,1)$ and $\gamma=1$ allowed in episodic settings. We use tabular Q-Learning with $\epsilon$-greedy action selection to learn an optimal policy. Q-Learning learns the optimal state-action-value function $Q^{\pi^*}(s,a) =\E_{\pi^*}[G_t|S_t=s,A_t=a]$ using a temporal difference update based on Bellman's equations,

\vspace{-.2cm}

\scriptsize
\begin{equation}\label{eq-qupdate}
Q_{t+1}(S_t,A_t) = Q_t(S_t,A_t) + \alpha \ TD_\gamma(S_t,A_t,R_{t+1},S_{t+1}),
\end{equation}
\normalsize 
where $\alpha$ is the learning rate and
\scriptsize 
\begin{equation*}
TD_\gamma(S_t,A_t,R_{t+1},S_{t+1})  = R_{t+1} + \gamma \max_{a}Q_t(S_{t+1},a) -Q_t(S_t,A_t).
\end{equation*}
\normalsize

In $\epsilon$-greedy action selection, the RL agent takes a random action with probability $\epsilon\geq 0$,  and with probability $1-\epsilon$, the agent takes action,  $\pi_{t+1}(a|s) = \argmax_a Q_{t+1}(s,a)$.  Standard Q-learning does not require modeling and  learning the reward function $r(s,a)$ or transitions $T(s'|s,a)$ explicitly, making it a `model-free' algorithm. 

\subsection{Missing Data in RL}\label{sec-missingdatarl}

Missing data in RL occurs when the agent receives incomplete information about the state or reward. The action is usually not missing since the agent decides this. We assume rewards are observed and focus on fully or partially missing states. Let $M_t$ be a binary vector of the same dimension as $S_t$ with $M_{t,j}=1$ if $S_{t,j}$ is missing and $0$ otherwise and let $\mathcal{M}$ be the set of all possible missingness vectors $M_t$. At time $t$, let $S_t^{obs}=\{S_{t,j} : M_{t,j}=0\}$ be the observed part of state $S_t$ and $S_t^{mis}=\{S_{t,j}:M_{t,j}=1\}$ be the missing part, where one of these could be $\emptyset$.

Missingness mechanisms are commonly classified into three cases \citep{Rubin1976}. Within a single time-step, these are (1) Missing Completely at Random (MCAR) where $M_t \indep S_t$, (2)  Missing at Random (MAR) where  $M_t\indep S_t^{mis} \mid S_t^{obs}$ but missingness may depend on observed values, and (3) Not Missing at Random (NMAR), the most challenging case where missingness depends on the values of the missing data (e.g., truncation). Missingness is a challenge both at the point of action selection, where the agent knows only $S_t^{obs}$ but the optimal action may depend on $S_t^{mis}$, and for learning,  where it becomes unclear how to do $Q$-Learning or other updates given only $(S_{t}^{obs},A_t,R_{t+1},S_{t+1}^{obs})$.  The sequential nature of RL also raises the possibility of missingness that depends on longer histories of states, actions, and rewards.

\subsection{Related Literature}

The standard multiple imputation (MI) recipe is to fit a Bayesian model of missing data $X_{mis}$ and parameter $\theta$ given observed data $X_{obs}$, draw $(X_{mis}^{(k)})_{k=1}^{K}$ from the posterior predictive distribution, calculate an estimate for each complete dataset, and combine them for estimation and uncertainty quantification \cite{RubinDonaldB1987Mifn, Schafer_1999}. In offline RL, where sequences $\{(S_{t}^{obs},A_{t},R_{t+1},S_{t+1}^{obs})\}_{t=1,...,T}$ have already been collected, sequential MI and other MCMC based methods are available, and having a human check output for convergence is feasible \cite{lizotte2008missing, Shortreed2011, Yamaguchi2020}.
The problem with applying these online  is that re-fitting models repeatedly can be prohibitively expensive, especially if the algorithms require human review.

A related RL approach applicable in principle but potentially infeasible online is the Partially-Observed MDP (POMDP), in which the agent maintains a \textit{belief state} (posterior distribution) over unknown parameters and latent states and learns a policy given the belief state \cite{Pascal_etal_2021}. Computational efficiency is a challenge for POMDPs, as solving them generally requires approximations to deal with its continuous state space \citep{Poupart2008,Shani2013}. %
Although standard POMDP pose fixed latent versus observed states, \citet{Wang2019} adapt them to the case of ``dynamically missing" and noisy observations. They provide Gaussian belief state approximations and transition models to make updates more tractable but also leverage deep neural networks, which can be costly to train. Their method requires a \textit{training phase} for learning a transition model, which means it is not fully online.

A key insight from the POMDP  is the need to propagate uncertainty in missing state space values. In this paper, we explore how well we can achieve this via imputation ensembles formed from simpler, more myopic, probabilistic imputation models. The resulting imputation pathways can represent different possibilities for the missing states and are combined for action selection. Ensembles have been used with success in other areas of RL, including ensembles that use different learning algorithms \cite{WieringHado2008} and ensembles of different Q functions, used to reduce variance and address maximization bias in Q-learning \cite{fausser2015neural, Chen_etal_2021, SuttonBarto2018, Ghasemipouretal_2022}. %

Ensembles also connect to a literature on model-based RL `rollouts' in which synthetic data, generated under some model, is incorporated into the process of RL learning to speed up learning \cite{Rajeswaren2017,janner2019mbpo, Ghasemipouretal_2022}. There, too, ensembles have been used to help prevent the learning algorithm from relying too much on any particular model and thereby learn policies which better generalize to the target domain. %
Like these methods, our approach requires learning some transition distribution and using them to generate synthetic data (imputations).  Although in our experiment, a saturated tabular transition model is feasible, more generally, some restriction to simpler models of the transition distribution will often be necessary in the online setting to avoid long learning times and computationally and data-intensive function approximator \citep{Rajeswaren2017} -- all the more so when learning is slowed by missingness in the observations.

There is also a connection to work on \textit{delayed} rewards and/or states,  where there is a need to impute or propagate uncertainty about what the current state may be given that we have not observed it yet and/or model the delay mechanism (e.g., \citet{agarwal2021, Chenetal2023, joulani2013online}). Some of these methods may be extendable to cases where some states are \textit{never} observed, but these methods do not address how to deal with \textit{partial} missingness where $S_{t+1}^{obs}$ is observed and might be predictive of $S_{t+1}^{mis}$. We focus on a case without any delays but future work might consider states where each component may be observed immediately, with a delay, or never.

Perhaps the closest approach to our own is that of \cite{ma2020discriminativeparticlefilterreinforcement},
who \textit{approximate the belief state} with a set of importance-weighted particles. While this approach is somewhat similar to our multiple-imputed states setup, one critical difference is that Ma et al.'s approach requires using a neural network to represent what they term a \textit{compatibility function} $f_{\text{obs}}(h_t, o_t)$, i.e., an unnormalized surrogate  of the generative observation model $p(o_t \mid h_t)$ (as as weighting tool), where $o_t$ is the observed state and $h_t$ is the true latent state. Technically, the \textit{compatibility function} only needs to capture aspects of the modeling setup that are relevant for the underlying RL task and does not need to be a true probability distribution. Nonetheless, our proposed method is designed to be less computationally-expensive (no surrogate neural network necessary, and updating only a portion of our framework parameters) and avoids potentially high-variance importance-weighting. In addition to proposing our ensemble method, we hope to provide a generally useful framework for thinking about the various design choices involved when working with missing data in RL.%

\section{Methods} \label{sec-methods} 

\subsection{Forming Imputation Ensembles}

\begin{figure}
    \centering
    \begin{tikzpicture}
        \node[rectangle,draw,minimum size=0.5cm,fill=blue!10,line width=1.2pt] (S1) at (3,-1) {$S_1$};
        \node[rectangle,draw,minimum size=0.5cm,fill=blue!10,line width=1.2pt] (S2obs) at (5,-1) {$S_2^{obs}\ \ $};
        \node[rectangle,draw,minimum size=0.5cm,fill=blue!10,line width=1.2pt] (S3obs) at (8,-1) {$S_3^{obs}\ \ $};
        \node[rectangle,draw,minimum size=0.5cm,fill=blue!10,line width=1.2pt] (S4) at (10,-1) {$S_4$};

        \node[rectangle,draw,minimum size=0.5cm,fill=white,line width=1.2pt] (S2mis) at (5,0) {$S_2^{mis,k}$};
        \node[rectangle,draw,minimum size=0.5cm,fill=white,line width=1.2pt] (S3mis) at (8,0) {$S_3^{mis,k}$};

        \node[rectangle,draw,minimum size=0.5cm,fill=white,line width=1.2pt] (B2mis) at (5,-2) {$S_2^{mis,k'}$};
        \node[rectangle,draw,minimum size=0.5cm,fill=white,line width=1.2pt] (B3mis) at (8,-2) {$S_3^{mis,k'}$};

        \draw[->, ultra thick] (S1) -- (S2obs);
        \draw[->, ultra thick] (S2obs) -- (S3obs);
        \draw[->, ultra thick] (S3obs) -- (S4);

        \draw[ultra thick] (S1) -- (B2mis);
        \draw[ultra thick] (S1) -- (S2mis);
        
        \draw[->, ultra thick] (S2obs) -- (S2mis);
        \draw[->,ultra thick] (S2obs) -- (S3mis);
        \draw[->, ultra thick] (S2obs) -- (B2mis);
        \draw[->,ultra thick] (S2obs) -- (B3mis);

        \draw[->, ultra thick] (S3obs) -- (B3mis);
        \draw[->, ultra thick] (S3obs) -- (S3mis);
        
        \draw[->,ultra thick] (S2mis) -- (S3mis);
        \draw[->,ultra thick] (B2mis) -- (B3mis);

    \end{tikzpicture}
    \caption{\footnotesize{Pathways of imputations with $S_1$ and $S_4$ fully observed. Past imputations affect future imputations, though not the observed parts of future states. Actions, which also affect and are affected by states and imputations, are not depicted.}}
    \label{fig-pathdiagram}
\end{figure}
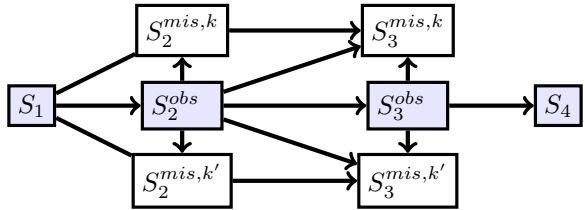

The idea of an \textbf{imputation ensemble} is to maintain $K$ different \textbf{imputation pathways} that contain different sequences of imputed missing values as in Figure \ref{fig-pathdiagram}. For each path $k$, given a previous state-action pair $(S_t^k,A_t)$ and the current $S_{t+1}^{obs}$, we draw $S_{t+1}^{mis,k}$ and set $S_{t+1}^{k} = (S_{t+1}^{mis,k},S_t^{obs})$. This forms a \textit{path} because for each $k$, we condition on the previous (possibly) imputed state $S_t^k$ when imputing  the next state. If the state is repeatedly missing, these paths could cumulate errors where imputing $S_{t}^{mis,k}$ incorrectly leads to an even more incorrect $S_{t+1}^{mis,k}$. Ensembles may be able to mitigate this by representing multiple plausible pathways, especially if their imputation errors are not too correlated \citep{fausser2015neural, breiman2001random}. 
To obtain these pathways and use them for Q-learning in RL, important design choices arise at three junctures: (1) Generating imputations (2) Learning, and (3) Action selection. A full procedure for our main method (i.e., multiple-imputation with synthetic-data fractional Q- and T-updates) is given in Algorithm \ref{algo:fractional_q_learning}. Variants of our methods discussed in the following subsections can be constructed similarly.

\subsubsection{Generating Imputations}\label{sec-transitions}

Probabilistically imputing $S_t^{mis}$ requires learning a model or models of the distribution $S_{t+1}^{mis}|S_{t+1}^{obs},S_t, A_t$. Model-based RL methods already involve estimating transition models $T(s'|s,a)$ but translating these to the conditional distributions $\hat{T}_t(s'_I|s'_{I^c},s,a)$ for $I\subseteq\{1,...,d\}$ when $s'$ is partially observed may not be immediate. In lower dimensional tabular cases, estimating transition probabilities can be done by accumulating counts. Regardless of approach, learning a transition model online while also imputing from it  raises a question of  whether to update the transition model using only observed data  or  also  \textit{synthetic imputed data}. In our experiments, we consider two options:

\vspace{-.3cm}
\begin{enumerate}
    \item \textbf{Conservative}: fit an imputation model that learns only from counts based on fully observed tuples $(S_t,A_t,S_{t+1})$, with update
\footnotesize
    \begin{align}\label{eq-tudpate-cons}
        \hat{T}_{t+1}(s'|a,s) &= \frac{n_{s,a,s'}^{(t+1,full)}}{n_{s,a}^{(t+1,full)}}  \hspace{.3cm}\forall s,a,s'
    \end{align}
    \normalsize 
    
    \item \textbf{Synthetic}: fit an imputation model that learns from observed and imputed (synthetic) data counts. In the tabular case, at time $t+1$, for $k=1,...K$, we update
    \footnotesize
    \begin{align}\label{eq-tudpate}
        n_{S_{t}^k,A_t,S_{t+1}^k}^{(t+1)} &=  n_{S_{t}^k,A_t,S_{t+1}^k}^{(t)} + \frac{1}{K} \\
        n_{S_{t}^k,A_t}^{(t+1)} &=  n_{S_{t}^k,A_t}^{(t)} + \frac{1}{K} \\
        \hat{T}_{t+1}(s'|a,s) &= \frac{n_{s,a,s'}^{(t+1)}}{n_{s,a}^{(t+1)}}  \hspace{.3cm}\forall s,a,s'
    \end{align}
    \normalsize 
    The fractional update ensures that single imputations have less weight than observed data and ensures proper normalization (Appendix \ref{sec-fracupdate}). %
\end{enumerate}

Option 1 risks throwing out partial observations but option 2 risks self-reinforcement where the data appear to support a certain model but were generated under that model. Other options, which we leave for future investigation, are (3) to fit $K$ different imputation models, one for each $(S_t^k,A_t,S_{t+1}^{k})$, which is more expensive but could lead to less correlated imputations, and (4) to fit only univariate models of the $d^{th}$ element of $s'$ given $s_{-d}',s,a$, possibly with cycling as in MICE \cite{vanBuuren2007,Azuretal2011}. This is less principled but may be a necessary compromise in higher dimensional settings.

Given missingness, the \textbf{Markovian assumption} of the MDP may no longer be satisfied by $S_t^{obs},A_t$ alone, meaning it could be beneficial (but more difficult) to account for the longer history of observations when imputing. The options above yield imputations  that are \textbf{myopic} because they ignore this fact and instead use the possibly imputed previous state as if it were observed. Even so, the Markovian nature of the underlying environment is helpful, as whenever we have fully observed $S_t$, the imputation pathways sync (Figure \ref{fig-pathdiagram}) and past states do become irrelevant, though past imputations may still impact the estimated transitions and  be embedded in $Q$ updates.

\subsubsection{Learning Updates}

Given multiple imputations we learn a single Q function by applying a version of Equation  \eqref{eq-qupdate} with a fractional learning rate sequentially for $k=1,...,K$. 

\scriptsize
\begin{equation}\label{eq-qupdate-frac}
Q_{t+1}^{(k)}(S_t^{k},A_t) = Q_{t+1}^{(k-1)}(S_t^{k},A_t) + \frac{\alpha}{K} \ TD_\gamma(S_t^k,A_t,R_{t+1},S^k_{t+1}),
\end{equation}

\normalsize 
where $TD_{\gamma}$ is calculated for $Q_{t+1}^{(k-1)}$ and letting $Q_{t+1}^{0}=Q_t^{(K)} =Q_t$. Though learning rate $\alpha$ needs to be tuned, the $\frac{1}{K}$ is not meaningless because for any given $\alpha$, it moderates the impact of imputations. When $S_t,S_{t+1}$ are fully observed so that for all $k$, $S_t^k =S_t^{obs},S_{t+1}^{k}=S_{t+1}^{obs}$, the full cycle of $K$ updates is, for $\alpha/K<1$, approximately equivalent to the overall Q learning update with learning rate $\alpha$ (Appendix \ref{app-frac-Qlearn}).
When $S_t,S_{t+1}$ have missingness and imputations vary over $k=1,...,K$, the updates are distributed and hence their impact diminished --  moreso the more the imputations disagree. This suggests it is useful for the variation in the imputations to be well-calibrated to the actual uncertainty as to the missing states. Another option would be to maintain $Q_{t+1}^{k}$ separate Q functions, doing one $\alpha$ rate update per pathway, and then combine them at point of action selection. We opt not to do this not only because it increases computational burden but because we suspect it would lead to more instability as each learned Q function would never be corrected by any of the alternative imputation pathways.

\subsubsection{Action Selection}

To better promote exploration of the state-action space, we select the next action at each timestep using a voting ensemble over the various imputed states, as opposed to an argmaxing approach. Formally, define policy $\pi_{t+1}(s) = \argmax_a Q_{t+1}(s,a)$. We then calculate $A_{t+1}^{k} = \pi_{t+1}(S_{t+1}^{k})$ for $k=1,...,K$ and select among the $A_{t+1}^{k}$ at random. In future work, we will explore alternative softmax and average Q function based action selection strategies.

%
\begin{comment}
In line with existing ensemble methods, for action selection, we consider multiple options for combining pathways to select an action: 
 
 \begin{enumerate}
     \item \textbf{Equal Voting}: Define policy $\pi_{t+1}(s) = \argmax_a Q_{t+1}(s,a)$. Calculate $A_{t+1}^{k} = \pi_{t+1}(S_{t+1}^{k})$ for $k=1,...,K$ and select among $A_{t+1}^{k}$ at random.
     %
     
     \item \textbf{Softmax Voting}: select among $A_{t+1}^{k}$ by calculating counts $c_1,...,c_{|\mathcal{A}|}$ of how often each action is voted for and selecting action $a$ with probability $\exp(c_a)/\sum_{a'}\exp(c_{a'})$  

     \item \textbf{Averaging}:  calculate the average $\bar{Q}_{t+1}(a) = \frac{1}{K}\sum_{k=1}^{K}Q_{t+1}(S_{t+1}^k,a)$ and  pick $\argmax_{a}\bar{Q}_{t+1}(a)$.

 \end{enumerate}

 \noindent In both cases, we also consider an $\epsilon$-greedy version where with probability $\epsilon$, an action is selected entirely at random. %
\end{comment}

\begin{algorithm*}[!t]
\begin{algorithmic}
\STATE {\bfseries Input:} $\mathcal{S},\mathcal{M},\mathcal{A}, \epsilon >0, \gamma \in [0,1), K \geq 1$, stop criterion \\
\STATE \textbf{\bfseries Init:} $Q_0(s,a) = 0$, $\hat{T}_0(s'|a,s)=0$,  for $a\in \mathcal{A},s,s'\in \mathcal{S}$\\
For $t=0$, observe $S_0^{obs}$. If $S_0^{mis}\neq \emptyset$, impute $S_0^k$ at random for $k=1,...,K$, pick $A_0$ at random, observe $R_1$\\
\WHILE{$t\geq 0$ and not [stop criterion]}
\item Observe $S_{t+1}^{obs}$
\FOR{$k=1$ \textbf{to} $K$}
\item  Draw $S_{t+1}^{mis,k} \sim \hat{T}_t(\bullet|S_{t+1}^{obs},S_t^k,A_t)$
\item Set $S_{t+1}^k = (S_{t+1}^{mis,k},S_t^{obs})$
\ENDFOR
\item Set $\pi_{t+1}(s)=\argmax_a Q_{t}(s,a)$ for $s\in \mathcal{S}$
\item Set $A_{t+1}^{k} = \pi_{t+1}(S_{t+1}^k)$
\item Select  $A_{t+1}\sim \text{Unif}(A_{t+1}^{1},...,A_{t+1}^{K})$. With probability $\epsilon$, instead select $A_{t+1}$ uniformly at random over $\mathcal{A}$
\item Observe $R_{t+1}$

\FOR{$k=1$ \textbf{to} $K$}
\item Set $Q_{t+1}^{(k)}(S_t^{k},A_t) = Q_{t+1}^{(k-1)}(S_t^{k},A_t) + \frac{\alpha}{K} \ TD_\gamma(S_t^k,A_t,R_{t+1},S^k_{t+1})$ \COMMENT{Fractional Q-update}
\ENDFOR
\FOR{$k=1$ \textbf{to} $K$}
\item Set $n_{S_{t}^k,A_t,S_{t+1}^k}^{(t+1)} =  n_{S_{t}^k,A_t,S_{t+1}^k}^{(t)} + \frac{1}{K}$
\item Set $n_{S_{t}^k,A_t}^{(t+1)} =  n_{S_{t}^k,A_t}^{(t)} + \frac{1}{K}$
\item Set $\hat{T}_{t+1}(s'|a,s) = \frac{n_{s,a,s'}^{(t+1)}}{n_{s,a}^{(t+1)}}  \hspace{.3cm}\forall s,a,s'$ \COMMENT{Synthetic-Data T-update}
\ENDFOR
\ENDWHILE
\end{algorithmic}
\caption{\color{black} Fractional Q- and T-learning with Multiple-Imputation (Synthetic-Data Transition Model Updating)}
\label{algo:fractional_q_learning}
\end{algorithm*}

\subsection{Baselines}

We compare the MI variants described above to three simple baseline methods:

\vspace{-.2cm}
\begin{enumerate}
    \setlength{\itemsep}{-0.2em}
    \item \textbf{Random Action}: given any missingness, take a random action with uniform probability over $\mathcal{A}$. Only update $Q$ for fully observed tuples $(S_t,A_t,R_{t+1},S_{t+1})$
    \item \textbf{Last Observed State--V1}: Given  the last fully observed state $S_t$, if $S_{t+k}^{mis}\neq \emptyset$, use $S_t$. Use imputation in Q update.
    \item \textbf{Last Observed State--V2} fill in each element of $S_{t+k}^{mis}$ using the most recently observed value of that dimension and keep $S_{t+k}^{obs}$. Use imputation in $Q$ update.
\end{enumerate}

\vspace{-.2cm}

These methods are fast, simple, and sometimes used in practice. 
 In contexts with low missingness, they may be reasonable, but each carries risks. Random actions may result in taking very risky actions too frequently. V2 of the last observed state option  may result in imputed states which are actually impossible but could work well if state dimensions are fairly independent and do not have reward interaction effects. V1 risks being more severely outdated.  Still,  outside of RL, there is some work suggesting diminishing returns to complex imputation, so it is important to understand whether simpler methods can also do well \cite{morvan2024imputation}.%

\subsection{Treating Missing as a State}\label{sec-mis-as-state}
We also consider treating `missing' as another state value so that the state space gains variants such as $(s_1,s_2,?), (s_1,?,?),$ and $ (?,?,?)$ and Q-learning can proceed as usual.  Under MCAR and MAR, this is in principle unnecessary and costly, as it expands the state space dimension and may discard information about likely values of the missing data. In some NMAR scenarios, it may be desirable, as \textit{the fact of missingness} can be a signal and the optimal action under missingness may differ from the one under more observed versions of the state. In Appendix \ref{sec-misasstate}, we illustrate in a contextual bandit example that the utility of the missing-as-state approach may also depend on the missingness rate -- performing well if missingness is very low or very high. Under low missingness, performing well is easy, and under high missingness, we hypothesize that the missing-as-state effectively reduces the dimension of the state space to a few common $?$ scenarios, making learning faster. %

\section{Experiments} \label{sec-experiment}

\begin{figure*}[!t]
    \centering
    \includegraphics[width = 4.2cm]{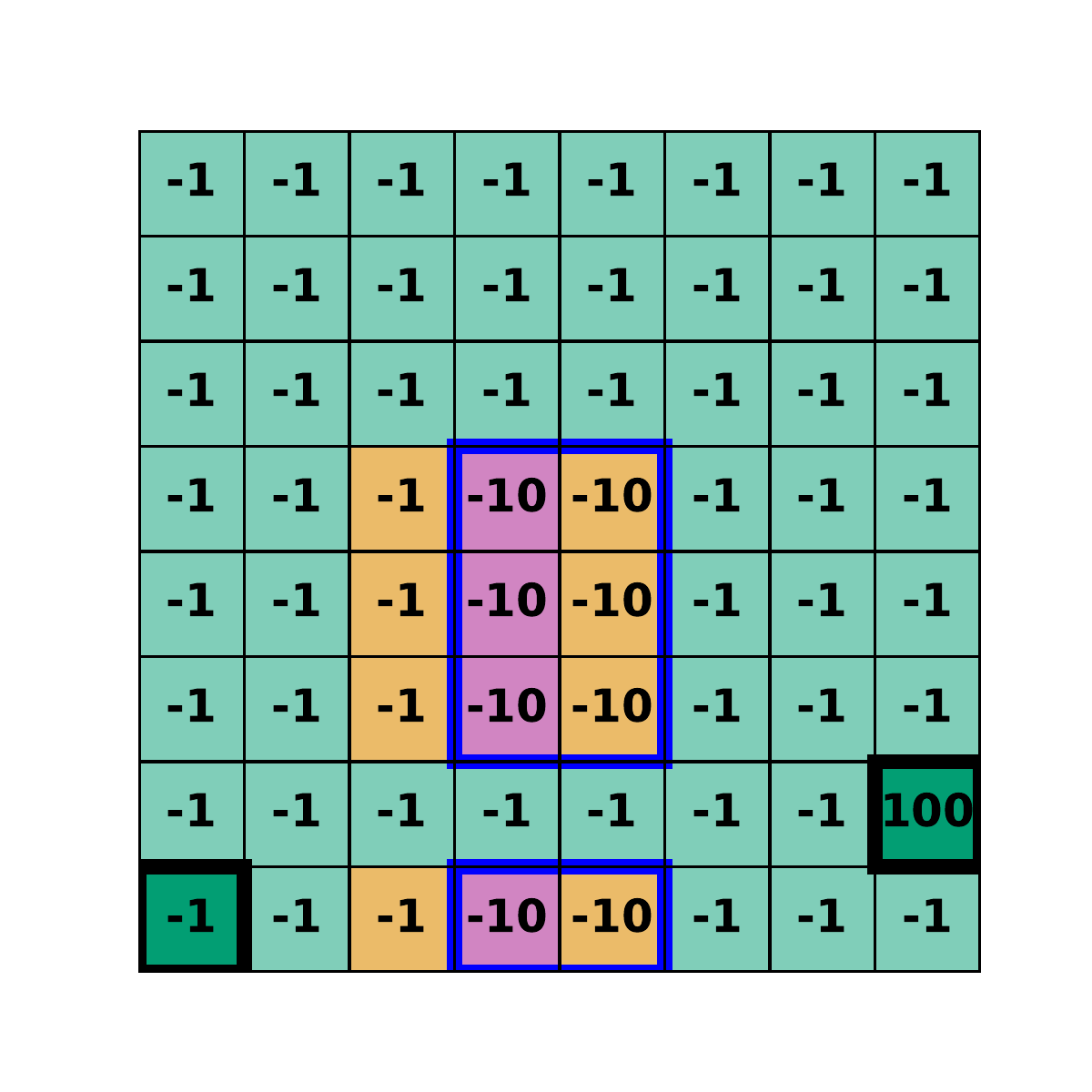}     \includegraphics[width = 4.2cm]{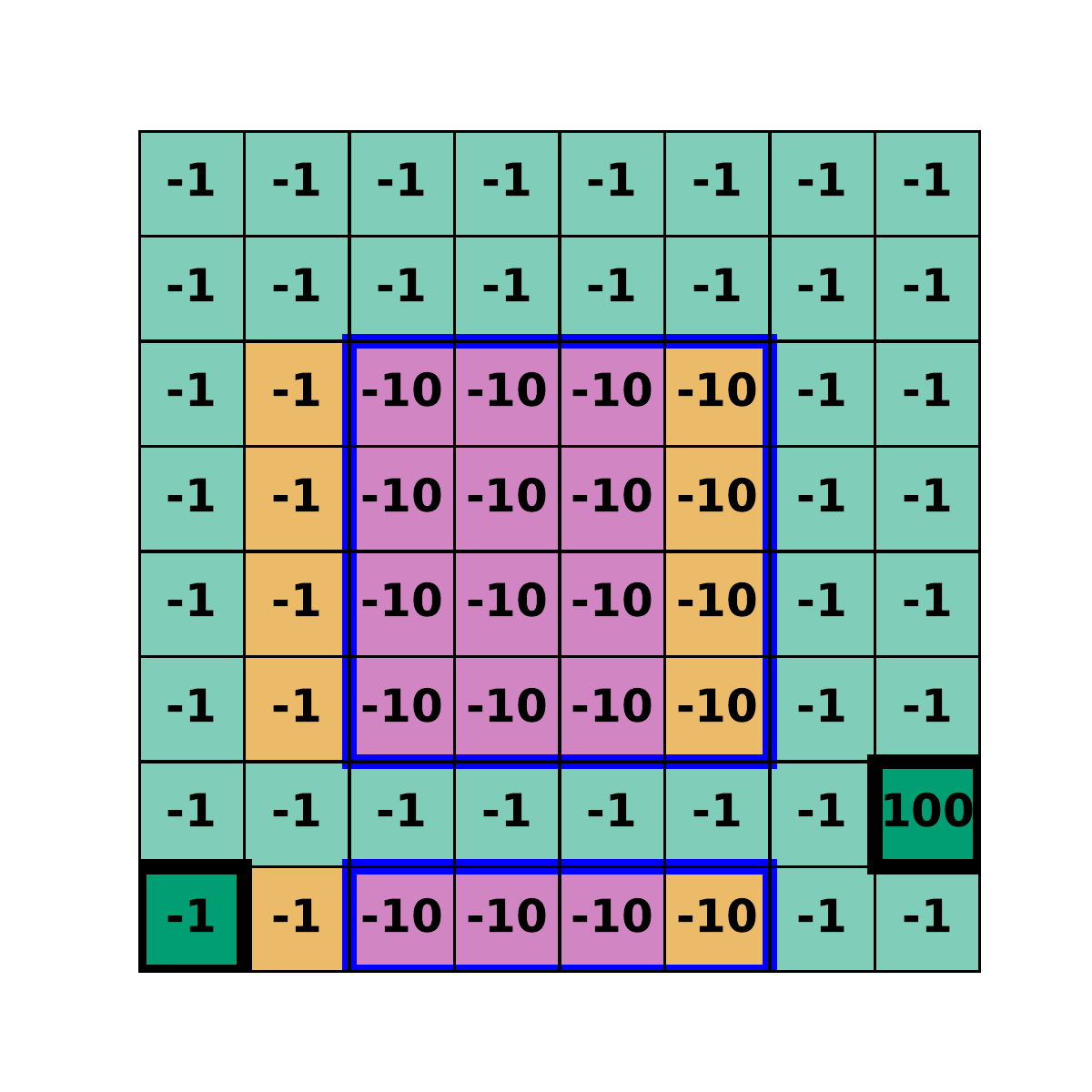}
   \includegraphics[width = 4.2cm]{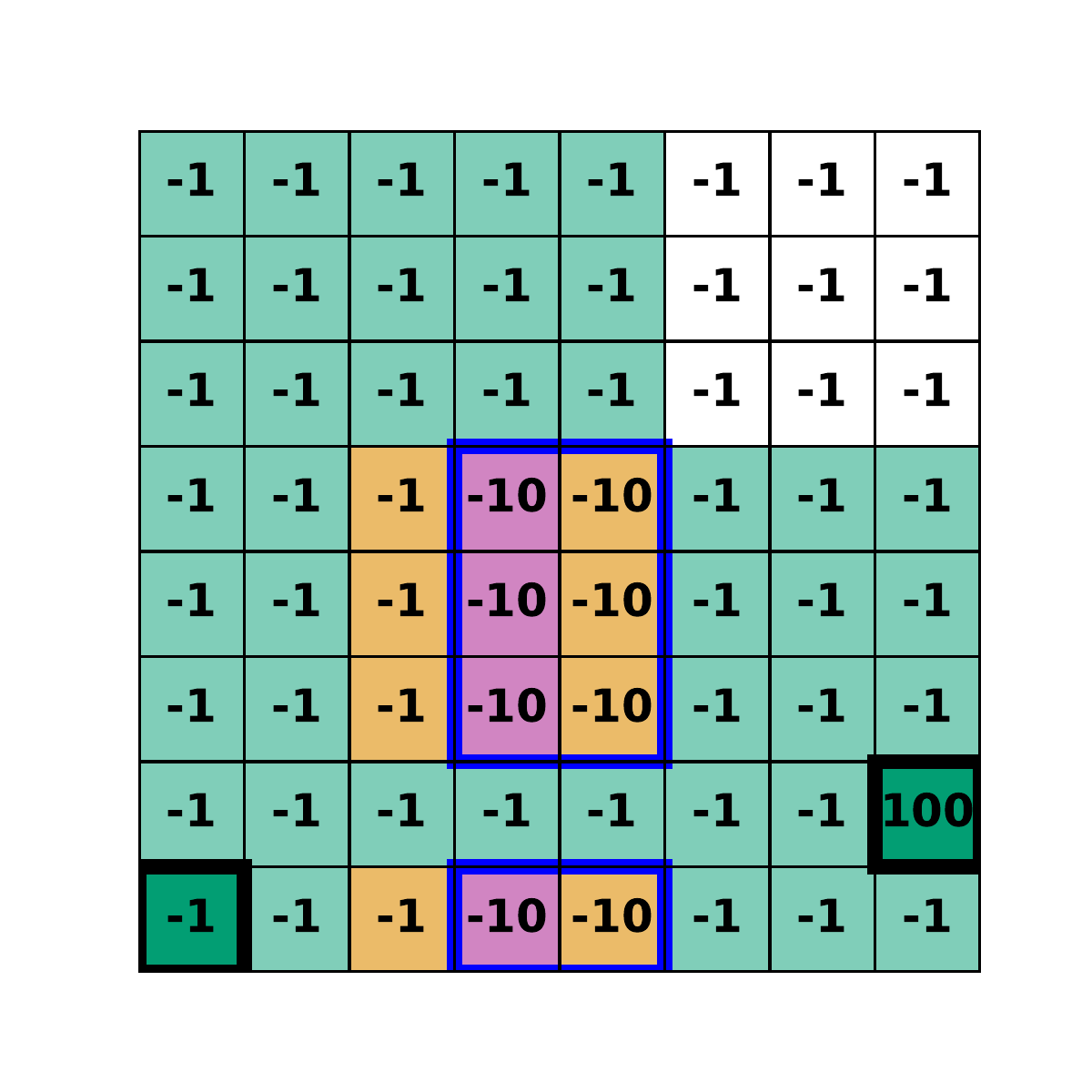}
   \includegraphics[width = 4.2cm]{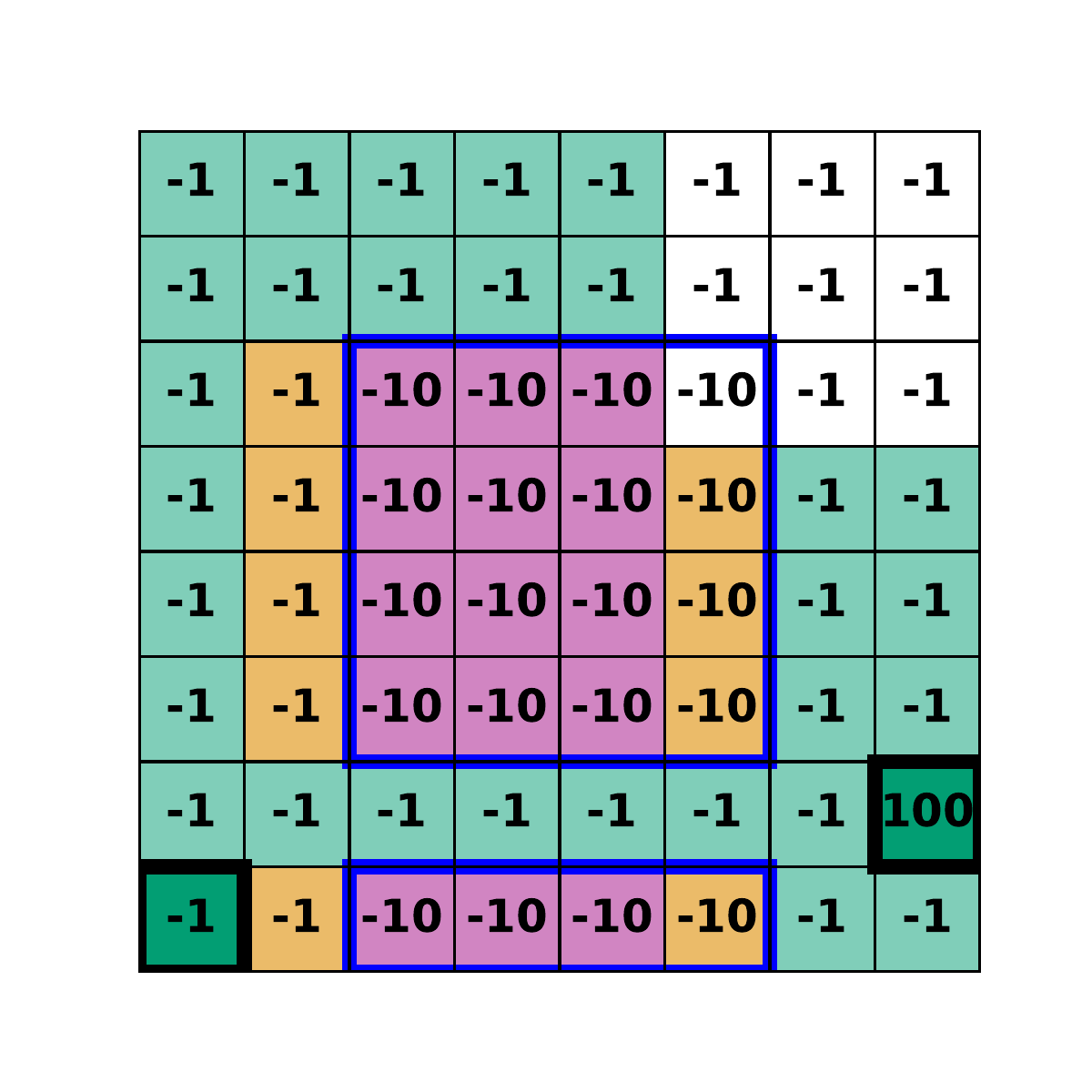}
    \caption{\footnotesize{Illustration of Grid Worlds. From left to right: no flooding nor fog; flooding and no fog; no flooding and fog; flooding and fog. The ``water" area  is outlined in bold blue. The dark green and outlined in black states are start (left) and terminal (right) states. \textbf{Regarding fog}: the 3x3 white region in the upper-right corner of the third and fourth panels indicate the presence of fog: states enshrouded in the fog have a higher probability of missingness under MFOG. To clarify, ``white" is \textit{not} a possible value for $S_{t,3}$. The underlying $S_{t,3}$ values of all states in the third and fourth panels are the same as the corresponding values shown in the first and second panels, respectively.}}
    \label{fig-gridworld}
\end{figure*}

\subsection{Grid World Testbed}

To test these ideas, we create a $8\times 8$ Grid World (Figure \ref{fig-gridworld}). The state space consists of all $(x,y,c)$ where $(x,y)$ are the location in the grid and $c \in \{\text{red},\text{orange},\text{green}\}$. To clarify, in other words, $S_{t,3}$ is a \textit{categorical color}.
The action space contains steps in all 8 directions, including diagonals (left, up, left-up etc.) and we experiment with allowing an additional action, a stay-in-place action. The agent always starts at a fully observed location --- the bottom-left corner of the grid --- and an episode ends when the agent reaches a terminal state on the right $(+100)$, which sends the agent back to the start. Steps incur a negative reward of $-1$, so the goal is to reach the terminal state as fast and often as possible. 

An added danger is stepping in a pond (-10 reward). Color signals  the danger of the current state and of taking a step right. Green means the agent is safe and moving right is safe. Red means the agent is not safe and moving right is not safe. Orange means either the agent is safe but moving right is dangerous, or the agent is in danger but moving right is safe (the agent must learn a correlation with the $x$ coordinate to use this color properly). Similar to Example 6.6 in \citet{SuttonBarto2018}, the shortest path is to cross a bridge but a safer path is to go around, especially since we make the state space stochastic by adding  Markovian \textbf{flooding} (Figure \ref{fig-gridworld}) that transitions with probability $p$ and \textbf{wind} which, with some probability, perturbs the action by $1$ unit within the $3\times 3$ grid of possible actions (e.g., up could become up+left).

We consider three missingness mechanisms. (1) In the \textbf{MCAR} case, each dimension ($x$, $y$, color) of $S_t$ is independently missing with $M_{t,j} \sim \text{Bernoulli}(\theta_j)$ for $j=1,2,3$. In our experiments, we set $\theta_j=\theta$ and vary this single parameter. For example, given $\theta$, the probability that at least one state is missing is $1-(1-\theta)^3$. (2) Under \textbf{MCOLOR}, the missingness rate depends on the color of the state, with higher rates for orange and red states (danger zones). If color itself is always observed, this is MAR whereas if color can also be missing, it is NMAR. (3) Under \textbf{MFOG}, the upper corner of the grid has a higher missingness rate than the rest of the grid (see Section \ref{sec-mfog}). This is NMAR.  

All code for implementing this grid world and our methods is available on request.

\subsection{Results}

\begin{figure*}[!t]
    \centering
    \includegraphics[width = 15.5cm]{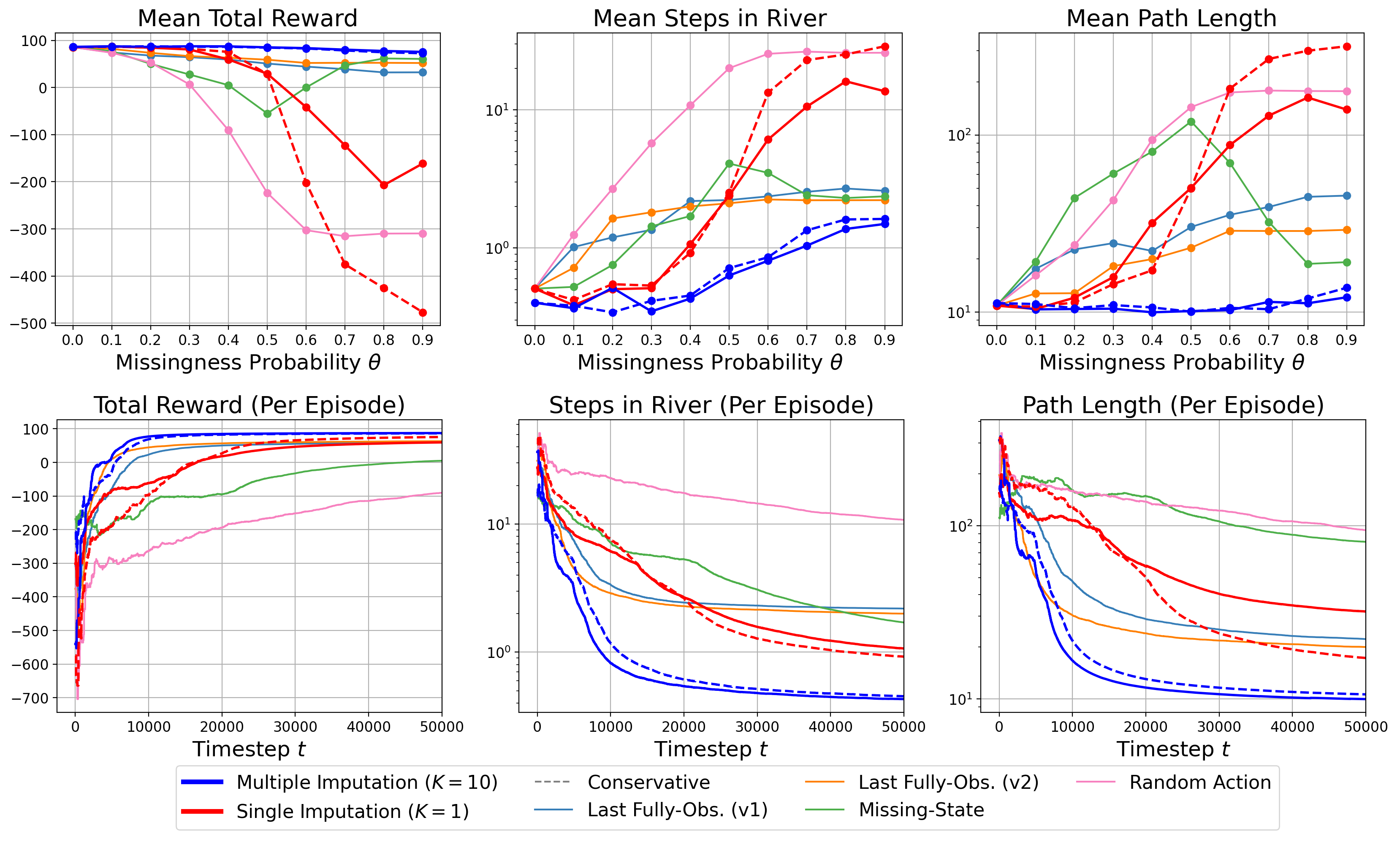} 
    \caption{\footnotesize{%
    \textbf{Top}: Comparison of mean performances over 50000 timesteps for all methods under MCAR with increasing missingness. \textbf{Bottom:} Performance over time for $\theta=0.4$. The y-axis in each case is the cumulative mean per episodes at time $t$ of the given metric. The environment is set to $P(\text{wind}) = 0.1$, $P(\text{flood\  transition}) = 0.1$. Each method shown for its best $\epsilon,\alpha,\gamma$, and action space option (stay in place allowed or not). Lines represent an average over 5 trials. Note the log scale for the center and right plots. The dashed blue (or red) line represent our multiple (or single) imputation model with ``conservative" learning of the transition model for imputations.}}
    \label{fig-mcar}
\end{figure*}

 We report the performance of each method for its best choice of $\epsilon \in \{0,0.05\}$, $\alpha \in \{.1, 1\}$, $\gamma \in \{0, 0.5, 1\}$ and whether to allow a stay-in-place action or not. We also set the flood transition probability and the probability of wind to $0.1$. For our MI/SI methods, we try  $K \in \{1, 5, 10 \}$.

\subsubsection{MCAR Results}

The top plot of Figure \ref{fig-mcar} shows that as expected, for low rates $\theta$ of MCAR missingness, all baselines, single imputation, and multiple imputation variants perform similarly. \textbf{As the missingness rate increases, all methods perform worse, but MI with $K=10$ outperforms single imputation and baselines across all three metrics and is quite robust}. Although the conservative T-update version learns slightly more slowly than the synthetic T-update version, the two are similar and both out-perform the rest.

For mean total reward, the missing-as-state and last-fully-observed-state baselines are only a bit worse, while in terms of steps in the river and path length, they are substantially worse (note the plots are on a log scale), suggesting a decline in ability to precisely navigate the water obstacle and bridge. Version 2 of the last-observed-state method is slightly better, which makes sense since it preserves observed information. Unsurprisingly, random action is quite bad. More interestingly, SI, though  better than baselines for $\theta \leq 0.4$, is horrible for higher missingness. For the synthetic T-update, this suggests concerns about path dependency may be well-founded. For the conservative T-update, this likely reflects that for high missingness, the transition matrix is  rarely updated. The bottom plot of Figure \ref{fig-mcar} shows performance over time for $\theta =0.4$, a setting where SI still slightly outperforms some baselines in the end. The plots highlight that even then, SI learns slower than the MI and some baselines, while random action and missing-as-state learn slowest of all. Further plots showing the learning dynamics at different $\theta$'s are available in Appendix \ref{appendix-mcar}.

The \textbf{missing-as-state baseline} (Section \ref{sec-mis-as-state}) displays an interesting $U$-shape relationship to $\theta$ in Figure \ref{fig-mcar}, performing worse as it increases but then better again for the highest $\theta$. We suspect that for low missingness, the expanded state space dimension makes learning slower while providing little gain since in MCAR, missingness is uninformative. But for high missingness, missing-as-state effectively lowers the dimension of the state space, as states such as $(x,?,?)$ become most common, and we have little joint $(x,y,c)$ information to help impute (though MI do still dominate).

\subsubsection{MCOLOR Results}

Appendix \ref{app-mcolor} and Figure \ref{fig-ColorFog} show performance under four MCOLOR settings. In each, we set the missingness rates for green $(\theta_{g})$, orange $(\theta_{o})$ and red $(\theta_{r})$ so that the ``danger signals" (orange and red) have higher missingness rates. Two settings are MAR with color never missing and the $(x,y)$ missing at $(\theta_{g}, \theta_{o},\theta_{r})= (.1,.2,.3)$ or $(.2,.4,.6)$. Two settings are NMAR with the same pattern only now color is also missing at that rate. The patterns from MCAR generally hold for MCOLOR as well, with MI consistently out-performing the others and again more robust as missingness increases. This dominance is again greater for higher missingness (see e.g.,  Figure \ref{fig-ColorFog}). Learning is slightly slower for the NMAR cases, but it is unclear if this is specific to this being NMAR or just to there being more missingness. Missing-as-state performs badly, but we did not push missingness high enough here to confirm the $V$-shape behavior.

\vspace{-.2cm}
\subsubsection{MFOG Results} \label{sec-mfog}

Appendix \ref{app-mfog} and Figure \ref{fig-ColorFog} show the performance of two MFOG settings, one where there is missingness in both the fog ($\theta_{in}=0.25$) and non-fog ($\theta_{out} = 0.1$) regions, and one with missingness only in the fog region with $\theta_{in}=0.5$. In the first case, MI still dominates. But when missingness is only in the fog region, missing-as-state performs on par with or better than MI with synthetic T-updates, though MI with conservative T-updates still performs well. This makes sense because in this setting, missingness signals location. Missing-as-state can learn this, while MI may go wrong, imputing the missing location to be outside the fog region because it has not often observed that region. The synthetic exacerbates the issue, as we discuss next.

\begin{figure*}[!t]
    \centering
    \includegraphics[width = 16cm]{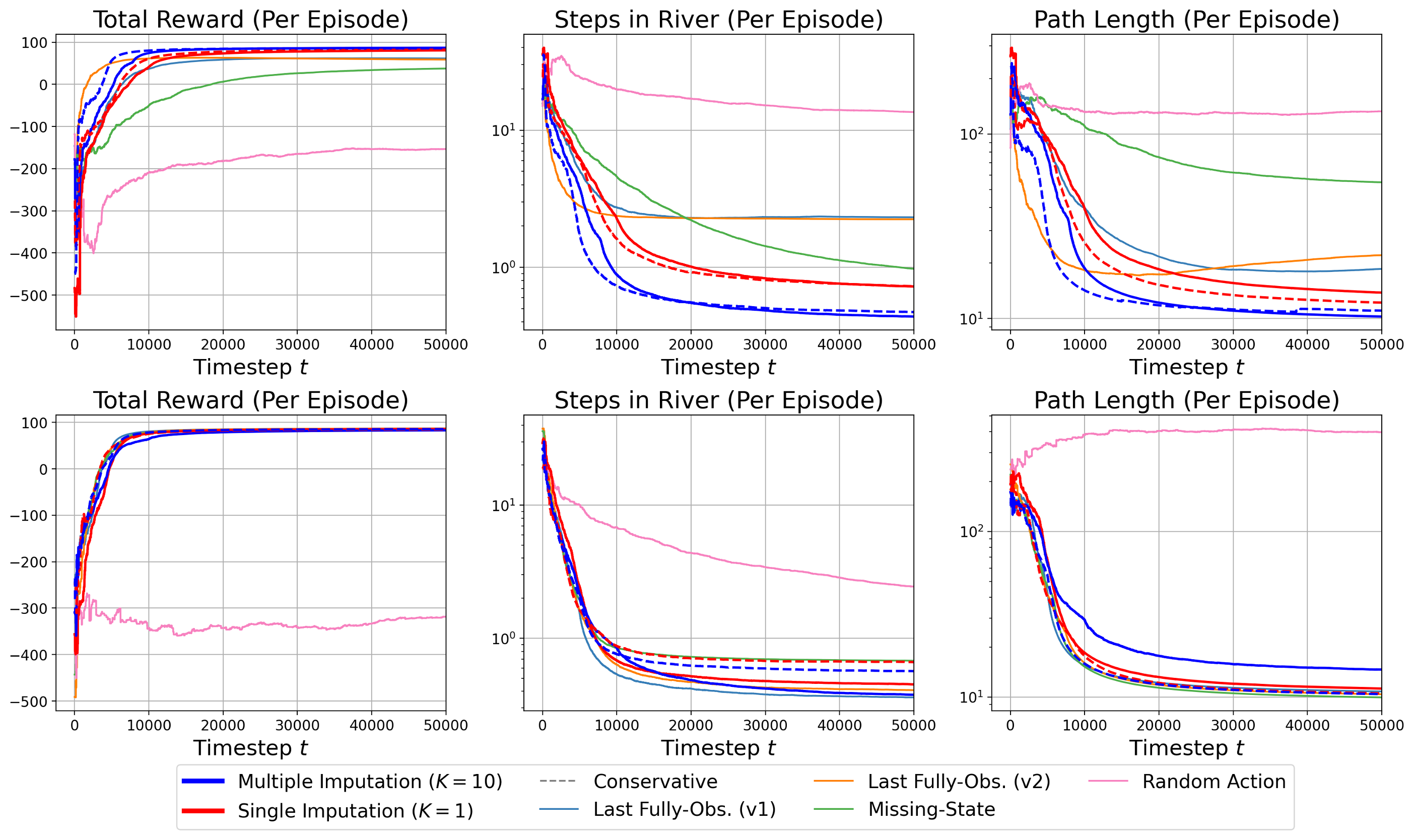} 
    \caption{\footnotesize{\textbf{Top}:  Performance over time for MCOLOR with missingness rates of $0.2$ (when in green), $0.4$ (when in orange), and $0.6$ (when in red) for all three of $x,y$ and color. \textbf{Bottom}: Performance over time for MFOG with missingness rate of $0.5$ inside the fog region and $0$ outside. The y-axis in each case is the cumulative mean per episodes at time $t$ of the given metric. The environment is set to $P(\text{wind}) = 0.1$, $P(\text{flood\  transition}) = 0.1$. Each method shown for its best $\epsilon,\alpha,\gamma$, and action space option (stay in place allowed or ot). Lines represent an average over 5 trials. Note the log scale for the center and right plots.}}
    \label{fig-ColorFog}
\end{figure*}

\subsubsection{Comparing Multiple Imputation Variants}

\textbf{Conservative vs Synthetic T-update}: in general, we observed that the conservative and synthetic update MI approaches performed similarly, suggesting that the synthetic T-update may not add much, perhaps because the imputations reflect no additional independent auxiliary information \citep{Meng1994, XieMeng2017}. For MCAR with higher missingness, the conservative update learned slightly slower but performance was similar. For MCOLOR, they were again similar, though conservative learned slightly faster. It is not clear if these rate differences, which are averaged over 5 trials, are significant. As described above, in the MFOG case with missingness only in the fog region, it appears  the use of synthetic updates was harmful. The example highlights the danger that, especially in NMAR settings, synthetic T-updates may reinforce incorrect data and weaken the signal of rare observed transitions in the missing region. This risk may outweight any benefit of leveraging some partially observed information for the updates.

\textbf{Choosing K:} In line with the existing MI literature, we find diminishing returns to increasing $K$, with $K=5$ outperforming $K=1$ and the gain for $K=10$ less dramatic (Appendix \ref{app-choose K}). However, higher dimensional state-action spaces could require higher $K$. There is nothing special about $K=10$.

\textbf{Mixing:} We did not find evidence that mixing was useful for our test cases. It made little difference in performance (Appendix \ref{app-choose K}). It may be that other features of our method already mitigate path dependence. Mixing could be more important in high dimensional settings where paths can diverge into more extreme territory.

\subsubsection{Stay in Place}\label{sec-stayinplace}

In Appendix \ref{app-stay-in-place}, we examine whether allowing a ``stay in place" (SIP) action makes a difference. We see no consistent difference in performance. As discussed below, stay in place actions could be beneficial in missing data contexts, but our algorithm does not leverage it optimally here, as reflected in the fact that the greedy algorithm simply learns not to use it.

\section{Discussion}
\label{sec-discussion}

The Grid World experiment provides an initial proof of concept suggesting that MI could be a tool for efficient online RL with missing data. \textbf{In the experiment, MI ensembles generally out-performed single imputation and simple baselines} in terms of maximizing reward, minimizing steps in the river, and minimizing path length. This held across missingness mechanisms and rates. The MI ensembles were more robust, performing well for longer as we increased the rate of missingness, and it learned faster. Superior performance generally held for both the conservative and synthetic versions of the T-update, with neither dramatically better than the other.  However, one experiment suggested the synthetic updates carries risks in NMAR settings. Examining different missingness mechanisms also helped surface the limitations and possible advantages of treating missingness as a state. Future work could explore combining the two approaches, for example by augmenting the state to include both imputations and information about the missingness structure.

These results are preliminary and the methods require further testing in a variety of settings to see whether they generalize, including to non-Grid World settings, infinite horizon settings with no terminal state, and more pathological NMAR scenarios. Of particular interest for future work is how performance varies with environment stochasticity. Stochasticity makes imputation harder but also makes it harder to take an optimal action so that wrong imputations could be less costly. On the other hand, missingness adds further stochasticity so correct imputation could be all the more important for eking out what little signal there is.

Though we focus on the low-dimensional tabular setting where $T$-probabilities and Q-functions can feasibly be captured in a table, the general procedure can be applied with different $Q$  and $T$ updates (e.g., model-based, e.g., incorporating knowledge of the missingness mechanism), and the design choices we highlight are generally relevant. We have also highlighted some additional design possibilities, such as incorporating multiple transition models into the assemble or adapting MICE to this setting to learn conditional transitions of the form $S_{t+1,j}\mid S_{t+1,-j},S_{t},A_t$ rather than full joint transitions.

As noted in Section \ref{sec-stayinplace}, the `stay in place' action made little difference in our experiment, but in general, this action represents another interesting avenue for further work. Staying in place represents an ``information gathering" action but not necessarily a safe or high reward one. Intuitively, given missing information, information-gathering could be a valuable thing to do, especially if we allow the agent to \textit{accumulate} information over time steps while staying in place.\footnote{That said, not all environments allow a ``staying in place" action as literally as in a Grid World.} More broadly, Q-learning and other standard algorithms are not designed to target informative actions or navigate when to act on partial information and when to wait. It would be interesting to consider how to combine information-theoretically-motivated algorithms such as those described in \citet{russo2016information, russo2018learning} with the MI ideas described here to deal with missingness. A challenge is that it is not just a matter of more fruitfully navigating the exploration-exploitation trade-off while exploring less over time -- information gathering becomes an action we need to  consider continually.

In the end, imputation methods are only useful if they also work on real data. Though it is hard to find real test cases with ground truth information on missingness, such examples might be artificially created on more realistic data. The key challenge is to figure out which missing data mechanisms are realistic. It is currently hard to obtain a comprehensive view of what kinds of missingness rates and dynamics are realistic and prevalent in real-life RL. In parallel to methodological developments like ours, it would be valuable to fill in that piece of missing metadata.

\bibliographystyle{ICML_stylefiles/icml2023}
\bibliography{references.bib}

\onecolumn 
\appendix

\section{Theory Appendix}

Throughout this section, we assume a tabular setting where the state space takes $C<\infty$ possible values and there are finitely many possible actions $|\mathcal{A}|<\infty$. A missing state is denoted by $s=?$. The goal  is to build intuition for our approach in more tractable cases.

\subsection{Justification for Fractional Updates}\label{sec-fracupdate}

Equation \eqref{eq-tudpate} in the main text uses a fractional count. Here, we argue this is needed to ensure proper normalization. Assume an MDP with transition $T(s'|s,a) \sim \text{Categorical}(p_{1|s,a},...,p_{C|s,a})$. Consider learning $p_{c|s,a}$ for a given $(s,a,c)$. We also let each be one dimensional so that a state is either missing ($s=?$) or not $(s\neq ?)$. Define the following variables:

\begin{table}[H]
    \centering
    \renewcommand{\arraystretch}{1.5}
    \begin{tabular}{c|l}
      $n_{s,a}^t$ & Count of fully observed $(s,a)$ after observing $S_t$ but before $A_t$ \\
      $n_{s,a,c}^{obs,t}$  & Count of fully observed $(s,a,c)$  after observing $S_t$ but before $A_t$    \\
      $n_{s,a}^{obs,t}$ & Count of fully observed $(s,a,s')$ for any $s'\neq ?$ after observing $S_t$ but before $A_t$\\
      $n_{s,a}^{mis,t}$ & Count of  $(s,a,?)$ after observing $S_t$ but before $A_t$  \\
    \end{tabular}
    \caption{For theoretical simplicity, these counts exclude any cases where the sequence is $(?, a, ?)$. They are defined to be calculable after observing $S_t$ but before taking action $A_t$. Note that $n_{s,a}^t = n_{s,a}^{obs,t}+n_{s,a}^{mis,t}$.}
    \label{tab:notation}
\end{table}

\subsubsection{Oracle Imputation}
 
Suppose that at step $t$, if $S_t=?$ and we need to impute some $S_t$ and take an action $A_t$ given a previous $(s,a)$, we are able to guess the missing state by receiving an independent draw from the true $\text{Categorical}(p_{1|s,a},...,p_{C|s,a})$ distribution, which we then use to take an action. Suppose we also want to estimate $p_{c|s,a}$, though for now we will not use that estimate for imputation. Let $X_{t,k} \simiid \text{Bernoulli}(p_{c|s,a})$ indicate if this draw is equal to $c$. Consider the estimate:

\begin{align*}
    \hat{p}_{c|s,a}^{t,(1)} &= \frac{n_{s,a,c}^{obs,t-1} + \frac{1}{K}\sum_{j < t}\sum_{k=1}^{K} X_{j,k}1(S_{j-1}=s,A_{j-1}=a, S_{j}=?) }{n_{s,a}^{t-1}}
\end{align*}

Where we want the estimate to be available before observing (or failing to observe) $S_t$, so the denominator counts only for completed tuples, for which the most recent is then $S_{t-2},A_{t-2},S_{t-1}$.

If missingness is MAR given $s,a$ with fixed probability $p_o$ this estimator is consistent for $p_{c|s,a}$ as $t\rightarrow \infty$ by the Law of Large Numbers\footnote{The standard LLN is applicable since under the MDP, conditional on $(s,a)$, the true $s'$ and imputed $s'$ are each drawn i.i.d.} and Slutsky,

\begin{align*}
    \hat{p}_{c|s,a}^{t,(1)}
    &= \frac{n_{s,a,c}^{obs,t-1}}{n_{s,a}^{obs,t-1}}\hat{p}_0 + \frac{\sum_{j< t} \sum_{k=1}^{K} X_{j,k}1(S_{j-1}=s,A_{j-1}=a, S_{j}=?)}{Kn_{s,a}^{mis,t-1}} (1-\hat{p}_0) \\
    &\xrightarrow{p} p_{c|s,a}p_o + p_{c|s,a}(1-p_o) = p_{c|s,a}
\end{align*}

where $\hat{p}_o = \frac{{n_{s,a}^{obs,t-1}}}{n_{s,a}^{t-1}}$. Notice that if we were to remove the $\frac{1}{K}$, this would converge instead to $p_{c|s,a}p_o + Kp_{c|s,a}(1-p_o)$. Of course,  we are getting free lunch that is impossible in practice: the ability to generate more  new data using the imputations. Consistency is unsurprising and the estimator will be more efficient than the observed-tuple-only estimator,

\begin{equation*}
    \hat{p}^{t,(2)}_{c|s,a} = \frac{n_{s,a,c}^{obs,t-1}}{n_{s,a}^{obs,t-1}},
\end{equation*}
which is also consistent under MAR. 

\subsubsection{Estimate Based Imputation}

Suppose now that we use the latest $\hat{p}_{c|s,a}^{t,(2)}$ to produce the imputations that we act on. This means the imputation distribution changes with $t$. This is in line with the conservative transition update mentioned in the text.

Imagine again that we also form an estimate $\hat{p}_{c|s,a}^{t,(3)}$ using these imputations in a similar manner to $\hat{p}_{c|s,a}^{t,(1)}$, though we do not use this estimate itself for imputation. That is, consider

 \begin{align*}
    \hat{p}_{c|s,a}^{t,(3)} &= \frac{n_{s,a,c}^{obs,t-1} + \frac{1}{K}\sum_{j<t}\sum_{k=1}^{K} X_{j,k}1(S_{j=1}=s,A_{j=1}=a,S_j=?)}{n_{s,a}^{t-1}},  \hspace{.5cm}   X_{jk} \simiid \text{Bernoulli}(\hat{p}_{c|s,a}^{j,(2)})
\end{align*}

Suppose we let $K=\infty$. Then the estimator becomes:

\begin{align*}
    \hat{p}_{c|s,a}^{t,(3)} 
        &= \hat{p}_o*\hat{p}_{c|s,a}^{t,(2)} + (1-\hat{p}_o)*\frac{\sum_{j<t}\hat{p}_{c|s,a}^{j,(2)}1(S_{j-1}=s,A_{j-1}=a, S_{j}=?) }{n_{s,a}^{mis,t-1}}
\end{align*}%

\noindent The right term is an average of a sequence of averages $\hat{p}_{c|s,a}^{t,(2)}$, each of which is unbiased and consistent for $p_{c|s,a}$ under MAR. We also have $\hat{p}_0 \xrightarrow{p}p_0$. Hence  $\hat{p}_{c|s,a}^{t,(3)} \xrightarrow[t\rightarrow\infty ]{p} p_{c|s,a}$. %
However, this estimator is less efficient than simply using $\hat{p}_{c|s,a}^{t,(2)}$. It only adds a noisier component based on previous noisier estimates, as the next theorem proves in the case where $p_o$ is known (estimating it should only add more variance) and we average over all previous $\hat{p}_{c|s,a}^{j,(2)}$ (averaging over fewer should also increase variance). 

\begin{theorem} \label{thm-variance-pcsa}
    Let $X_1,...,X_t \iid \text{Bernoulli}(p)$ with $\sigma^2 := p(1-p)$ , let $p_o$ be known,  and let $\hat{p}_j$ be the average of $X_1,...,X_j$, with variance $\frac{\sigma^2}{j}$. Define:

\[\tilde{p}_t = p_o \hat{p}_t + (1-p_o) \frac{\sum_{j< t}\hat{p}_t}{t}\]

Then as $t\rightarrow\infty$, $\frac{V(\tilde{p}_t)}{V(\hat{p}_t)} \rightarrow \infty$.

\end{theorem}

\begin{figure*}[h]
\begin{proof}[Proof of Theorem \ref{thm-variance-pcsa}]
This follows immediately from from the fact that $\hat{p}_t$ and $\tilde{p}_t$ are both unbiased estimators based on i.i.d. observations $X_1,...,X_t$ and $\hat{p}_t$ is the unique minimum variance unbiased estimator (MVUE). 
\end{proof}
\end{figure*}

This is related to the result in \cite{Meng1994} where, if no additional auxiliary information is injected into the imputation model, the multiple imputation based estimator is simply  `incomplete data estimator' as $K\rightarrow \infty$ and not give an advantage. In our case, even for $K=\infty$, $\hat{p}_{c|s,a}^{t,(3)}\neq \hat{p}_{c|s,a}^{t,(2)}$ because $\hat{p}_{c|s,a}^{t,(3)}$ is formed from imputations that were done sequentially. If at each time $t$, new imputations were drawn using $\hat{p}_{c|s,a}^{t,(2)}$ for all missing states at all previous times, then $\hat{p}_{c|s,a}^{t,(3)}$ at $K= \infty$ would simply be $\hat{p}_{c|s,a}^{t,(2)}$.

\subsubsection{Recursive Imputation and Estimation}

A third option most in line with the `synthetic' data approach in Section \ref{sec-transitions} is to use the form of $\hat{p}_{c|s,a}^{t-1,(1)}$ to draw imputations and then use those imputations as part of $\hat{p}_{c|s,a}^{t,(1)}$.

 \begin{align*}
    \hat{p}_{c|s,a}^{t,(4)} &= \frac{n_{s,a,c}^{obs,t-1} + \frac{1}{K}\sum_{j<t}\sum_{k=1}^{K} X_{j,k}1(S_{j=1}=s,A_{j=1}=a,S_j=?)}{n_{s,a}^{t-1}},  \hspace{.5cm}   X_{jk} \simiid \text{Bernoulli}(\hat{p}_{c|s,a}^{j,(4)})
\end{align*}

where for $K=\infty$, the estimator becomes

\begin{align*}
    \hat{p}_{c|s,a}^{t,(4)} 
        &= \hat{p}_o*\hat{p}_{c|s,a}^{t,(2)} + (1-\hat{p}_o)*\frac{\sum_{j<t}\hat{p}_{c|s,a}^{j,(4)} 1(S_{j-1}=s,A_{j-1}=a, S_{j}=?)}{n_{s,a}^{mis,t-1}}
\end{align*}

This creates a recursive relationship with complicated dependencies and should have even higher variance. In any case, it is still clear that the $\frac{1}{K}$ is necessary to ensure proper normalization and avoid the imputations dominating the observed instances.

\subsubsection{Takeaways}

We have shown that in this simple setting, fractional $K$ updates are necessary normalization. However, estimate of the transition distribution based on imputations does not confer advantages compared to estimating transitions using only the observed tuples. This makes sense, as the imputations do not inject new independent information. In line with \cite{Meng1994}, the point where we might expect imputations to become beneficial is when they allow us to incorporate additional information via partial observations we might otherwise ignore.  The next section explores this.

\subsection{Does Imputation Have Benefits for Learning? Comparison to Missingness as State}\label{sec-misasstate}

We suspect that the main benefit of imputation is not related to the missing state (which it may not always recover) but the way that it makes it possible to learn from partially observed states and their associated action-rewards. On the other hand, encoding missingness as a state and using states like $(s_1,?,s_3)$ also makes it possible to use partial information. The difference from imputation is that the missing-as-state approach does not leverage any correlations, learned from the longer history, that indicate what $s_2=?$ might represent and its consequences for the reward. Then again, if $s_2$ is very strongly correlated with $s_1,s_3$, it may be redundant information. So when should we prefer to impute versus treat missingness as a state? ( leaving aside for now the possibility of doing both)

In this section, we explore imputation and missing-as-state in a simplified contextual bandit set-up for a two-dimensional state space $(C,S)$  and action space $|\mathcal{A}|<\infty$. We assume $C\in \mathcal{C}$ is always observed and $S\in \mathcal{S}$ is sometimes missing. The contextual bandit simplifies theory  because the sequence of states $(C_{t},S_{t})$ is i.i.d., but we expect it to still reflect some of the dynamics of the MDP. Actions $A_t$ at each time step are taken according to some learning algorithm which may in general depend on the entire state-action-reward history $((C_j,S_j),A_j,R_j)_{j<t}$. We compare the following options.
\begin{enumerate}
    \item Impute $S$ using a learned distribution $p(s\mid c)$
    \item Add the states $\{(c,?):c\in \mathcal{C} \}$ to the state space
\end{enumerate}

\subsubsection{Building Intuition Through Four Scenarios}\label{scenarios}

 Which one is more advantageous depends on how well we can learn to predict $S$ using $C$ based on observed examples. Consider the following scenarios.

\begin{enumerate}
    \item \textbf{Highly predictive and MAR}: If $S$ is MAR and $C$ is very predictive of $S$, then (1) should perform better than (2) at least in terms of faster learning. To see this, imagine the extreme case where a certain $C=c$ always leads to $S=x$. On the one hand, then $(c,?)$ becomes a proxy for $(c,x)$, so it might seem like imputing $x$ versus using $?$ are equivalent. However, in the impute case, we get to pool the observe instances of $(c,x)$ \textit{and} the impute instances $(c,x)$ together and hence leverage them together when estimating expected rewards. In the missing as state $?$ instance, we separately tabulate the $(c,x)$ instances and the $(c,?)$ instances and consider the expected reward for each separately, meaning slower learning.

    \item \textbf{No predictive value and MAR}: As an opposite extreme, suppose MAR but $C_t\indep S_t$. In this case, we should not be able to learn any useful imputation model, so doing imputation only adds noise to our estimates of Q functions. Considering state $(c,?)$ is no panacea either, but now, these cases do not add noise to our Q estimates for fully observed cases. It makes sense to separate $(c,?)$ where the best we can do is learn the action that is best on average given $c$. 
   
    \item \textbf{NMAR but $?$ a deterministic signal}: The other extreme where missing-as-state would be the clear better option is if, given a certain $c$, whenever $S$ is missing, it takes a certain value $x$ that it never or rarely takes when $S$ is observed. In this case, the state $(c,?)$ is actually equivalent to $(c,x)$ and since $(c,x)$ is never or rarely directly observed, we do not really increase the state space dimension. Moreover, we avoid imputing values of $S$ based on observed $(c,s)$ which will be very off. This scenario is an extreme version of a missing not at random case.

    \item \textbf{General NMAR}: In general, if $C$ is predictive of $S$ but $S$ is NMAR, then the question of which approach is better \textbf{ depends.} We will learn a skewed imputation distribution and the question is whether that is a problem.  Consider an example where  $C=c$ implies $S\in \{x,y\}$. Suppose, however, that missingness is systematically higher if $S=y$. This means the observed $x,y$ rates are skewed and the imputation model too often imputes $(c, x)$. Yet this only a problem if the mean rewards for $(c,x), a$ and $(c,y),a$ and/or the ordering of these means rewards over $a$ are very different.  The fact that ultimately, we only need to pick an optimal action creates some buffer for bias created by incorrect imputations. On the other hand, using $(c,?)$ would lead to estimating a a pooled average of the rewards under $x$ and under $y$ together but without any ability to leverage fully observed cases. This will keep the estimated rewards for observed $(c,x)$ and $(c,y)$ cleaner but slow learning for  $(c,?)$.  Still, in NMAR scenarios, $(c,?)$ will sometimes  be a safer option.

\end{enumerate}

\subsubsection{Further Set-up and Notation}

In what follows, imagine we fix a value of $C$ throughout and let the $t$ indexing below be only over the observed $C=c$ cases. Because this is fixed throughout, we ignore $C$ in the notation below below, though it is still relevant in the sense that if $p_s$ below is the same for every $c$, there is no predictive advantage of leveraging observed information $C$. Every independence assumption below must be true conditional on every value of $C=c$ and we assume $C_t\indep C_j$ for $t\neq j$. We will suppose that we have some sequence of imputations $\hat{S}_t$ that satisfy an unconfoundedness assumption but are otherwise arbitrary.

\
\begin{table}[h!]
\centering
\begin{tabular}{ll}
\hline
\textbf{Notation} & \textbf{Description} \\
\hline
$p_s = \Pr(S = s \mid C=c)$ & Transition probability for $s \in \mathcal{S}$ \\
$x,y$ & Two distinct values in $\mathcal{S}$ \\
$M_t = 1(S_{t} = ?)$ & Missingness indicator at time $t$ \\
$M_{1:t} = (M_1, \dots, M_t)$ & Missingness history up to time $t$ \\
$A_{1:t} = (A_1, \dots, A_t)$ & Action history up to time $t$ \\
$S_{1:t} = (S_{1}, \dots, S_{t})$ & Sequence of $S$ states up to $t$ (observed or not) \\
$S_{1:t}^{obs}$ & Same as $S_{1:t}$ but with unobserved entries replaced by $?$ \\
\hline 
$\hat{S}_{t}$ & Imputation, $\hat{S}_j$ if missing, $?$ otherwise  \\ 
$\hat{S}_{1:t}$ & Imputation sequence up to time $t$ \\
$\hat{S}_{1:t}^k$ & A single chain of imputations \\
$\hat{S}_{1:t}^{1:K}$ & The set of all $K$ imputation chains \\ 
\hline
\end{tabular}
\caption{Summary of notation. Also, for any sequence, let a superscript with ${-j}$ denote it with the $j^{th}$ element removed.}
\end{table}

\textbf{Assumptions}

\begin{enumerate}
    \item \textbf{Contextual bandit}:  $S_t\indep S_{j}$ for all $t\neq j$ 

    \item \textbf{Rewards Observed}: the reward $R_t$ at time $t$ is an independent draw from some distribution $p(r\mid C_t=c,S_t,A_t)$

    \item \textbf{Missingness mechanism (MAR)}:  $M_t|C=c\iid Bernoulli(p_m)$ for some fixed $p_m \in (0,1)$. Note that this implies
$M_t \indep R_t(1),...,R_t(|\mathcal{A}|),S$.  Missingness may impact the action taken. We allow $p_m$ to be a function of $C$ in general but if not, this is also MCAR.
    
    \item \textbf{Unconfounded actions}: $A_{1:t} \indep R_{j}(1),...,R_{j}(|\mathcal{A}|)   \ \mid \ S_{j}=s$. This says that aside from $S_j$, there are no confounding variables affecting both the reward at time $j$ and the action at any past, present, or future time. This includes $A_{t} \indep R_{t}(1),...,R_{t}(|\mathcal{A}|)   \ \mid \ S_{t}=s$, which says that  given the current state, the action selection at time $t$ is independent of the potential outcomes at time $t$. %

\item \textbf{Action under missingness}: If $M_t = 1$, actions are independent of the true $S_t$; $A_t \indep S_{t} | M_t=1$ but may depend on the imputation $\hat{S}_{t}$. If $M_t=0$, the action is  independent of $\hat{S}_{t}$ (which we can imagine being drawn but ignored).

\item \textbf{Unconfounded imputations}: the imputation $\hat{S}_{t}$ at time $t$ is drawn at random using some probabilities  $\hat{p}_s$ which may change with $t$ and hence depend on the history or may be fixed throughout. The case where $\hat{p}_s = p_s$ (know the true probabilities) is a special case. The imputation may affect the action taken in the present (or future) but is independent of the current missing state and any potential outcomes. Imputations affect observed rewards only via the actions taken. 
\begin{align}
\hat{S}_t &\indep S_t \\
\hat{S}_{1:t} &\indep R_t(1), \dots, R_t(|\mathcal{A}|) \mid S_t \\
\hat{S}_{1:t} &\indep R_{1:t} \mid A_{1:t}, S_{1:t}
\end{align}

\item \textbf{Unobserved State Independence}: Given $M_j=1$, the true unobserved state $S_j$ is independent of all actions taken at time $t$ or any other time. $S_t \indep A_{1:k}|M_t=1   \hspace{.5cm} \text{ for any } k\geq 1$ It is also, as implied by the assumptions above, independent of all other states, imputations, missingness indicators, and observed rewards at times other than  $j$.

\item \textbf{SUTVA}: $R_i=R_i(A_i)$ and no interference, meaning the reward does not depend on the actions at any other times.
\end{enumerate}

\subsubsection{Estimators and Action Selection}

In a contextual bandit, we want to learn the mean rewards $\mu_{s,a} := \E[R(a)|S=s]$ for $s\in \mathcal{S},a\in \mathcal{A}$. Given these estimates, the greedy strategy is to pick the action with the highest estimated $\mu_{s,a}$. We ignore $\epsilon$-greedy steps for now, though the estimators below could also be used in that case. In a fully observed contextual bandit, each context produces an entirely separate bandit problem. If the state space can be missing, however, we face uncertainty as to which context we are in. Define the following estimators for all $s\in \mathcal{S}$:

\begin{align*}
    \mu_{s,a}^{{\color{blue}obs},t} &= \frac{\sum_{j=1}^{t}R_j1(S_{j}=s, A_{j}=a, M_j=0)}{\sum_{j=1}^{t}1(S_{j}=s, A_{j}=a, M_j = 0)}   \hspace{1.5cm} \mu_{?,a}^{{\color{blue}obs},t} = \frac{\sum_{j=1}^{t}R_j1(A_{j}=a, M_j=1)}{\sum_{j=1}^{t}1(A_{j}=a, M_j = 1)}  \\
    \\
  \\
  \mu_{s,a}^{{\color{blue}imp},K,t} &=  \frac{\sum_{j=1}^{t}R_j 1(S_{j}=s,A_{j}=a, M_j =0) + \frac{1}{K}\sum_{k=1}^{K}R_j1(\hat{S}_j^k=s,A_{j}=a, M_j =1)}{\sum_{j=1}^{t}1(S_{j}=s,A_{j}=a, M_j=0) + \frac{1}{K}\sum_{k=1}^{K}1(\hat{S}_j^k=s,A_{j}=a, M_j =1)}\\ 
\end{align*}

$\mu_{s,a}^{{\color{blue}obs},t}$  and $ \mu_{?,a}^{{\color{blue}obs},t}$ are a pair. Together, they only use the mean reward for observed cases and lump the rewards for unobserved cases into a separate estimator. In contrast, $ \mu_{s,a}^{{\color{blue}imp},K,t}$ additionally fractionally assigns rewards for imputed states to the mean for that imputed states. We consider three options for action selection. 

\begin{align}
   A_{t+1} &= \arg\max_{a} \mu_{s,a}^{{\color{blue}obs},t}  && \text{ where } s = \hat{S}_{t+1}M_i + S_{t+1}(1-M_i) \in \mathcal{S}\label{action-cons}\\ 
    A_{t+1} &= \arg\max_{a} \mu_{s,a}^{{\color{blue}obs},t}  && \text{ where } s = S_{t+1} \in \mathcal{S}\cup \{?\ \label{action-M}\\ 
   A_{t+1} &=  \arg\max_{a} \mu_{s,a}^{{\color{blue}imp},K,t}  && \text{ where } s = \hat{S}_{t+1}M_i + S_{t+1}(1-M_i) \in \mathcal{S} \label{action-MI}
\end{align}

The first option simply uses the mean from fully observed cases as if in the context indicated by the imputation. The second uses the mean that includes the imputations. The third ignores the imputation and uses the mean learned specifically for the $?$ case when needed. When $S_{t+1}$ is fully observed, $s$ is the same for all three but the means used may still be different.

\subsubsection{Results}
Below, we derive the means and variances of the mean estimators above, showing that under the assumptions above

\begin{align*}
    \E[\mu_{s,a}^{{\color{blue}obs},t}] &= \mu_{s,a}\\ 
    \E[\mu_{?,a}^{{\color{blue}obs},t}] &= \mu_{a} &&\text{ where } \mu_a = \sum_{s}\mu_{s,a}p_s\\
    \V(\mu_{s,a}^{{\color{blue}obs},t}) &= \sigma_{s,a}^2\E\left[\frac{1}{n_{s,a,o}}\right] && \text{ where } n_{s,a,o} = \sum_{j=1}^{t}1(S_j = s, A_j=a, M_j=0) \text{ and }\sigma_{s,a}^2=\V(R(a)|S=s)\\
     \V(\mu_{?,a}^{{\color{blue}obs},t}) &= \sigma_a^2\E\left[\frac{1}{n_{a,m}}\right] && \text{ where } n_{a,m} = \sum_{j=1}^{t}1(A_j=a, M_j=1) \text{ and } \sigma_a^2 = \V(R_j(a))\\ \E[\mu_{s,a}^{{\color{blue}imp},t,K}] &= \mu_{s,a}+  (\mu_a - \mu_{s,a})\E\left[ \frac{\hat{n}_{s,a,m}^K}{n_{s,a,o} + \hat{n}_{s,a,m}^K}\right]  && \text{where } \hat{n}_{s,a,m}^K=\frac{1}{K}\sum_{j=1}^{t}\sum_{k=1}^{K} 1(\hat{S}_j^k=s,A_{j}=a, M_j =1)\\
\end{align*}
\begin{align}
    \V(\mu_{s,a}^{{\color{blue}imp},t}) &=  \sigma_{s,a}^2\E\left[\frac{n_{s,a,o}}{(n_{s,a,o}+\hat{n}_{s,a,m}^K)^2}\right] + \sigma_a^2\E\left[\frac{\sum_{j=1}^{t}(\hat{p}_{js}^K)^2 1(A_{j}=a, M_j =1)}{(n_{s,a,o}+\hat{n}_{s,a,m}^K)^2} \right]  +  \left(\mu_{s,a}-\mu_{a}\right)^2\V\left(\frac{n_{s,a,o}}{n_{s,a,o} + \hat{n}_{s,a,m}^K}\right)
\end{align}

These results clearly indicate that the imputation based estimator is biased while the estimators that use only the fully observed cases are not. However, there can be variance benefits when $C=c$ is strongly predictive of the value of state $S$.

First, consider the extreme where there exists an $s^*$ such that $p_{s^*}=1$ as in the first scenario in Section \ref{scenarios}. Then, even if we impute using observed proportions,

 \begin{equation*}
     \hat{p}_{s^*} =  \frac{\sum_{j=1}^{t}1(S_{j}={s^*})(1-M_j)}{\sum_{j=1}^{t}(1-M_j)}
 \end{equation*}

we always have $\hat{p}_{s^*} = p_{s^*} = 1$, all imputations are correct, and multiple imputations are redundant (so let $K=1$). Since we never observe $s\neq {s^*}$, the mean reward estimators for $s^*$ are unbiased, with variances

\begin{align*}
     \V( \mu_{{s^*},a}^{obs,t})  &= \sigma_{{s^*},a}^2\E\left[\frac{1}{n_{{s^*},a,o}}\right] && n_{{s^*},a,o} = \sum_{j=1}^{t} 1(S_j = s^*, A_j=a, M_j = 0)\\
     \V( \mu_{?,a}^{{\color{blue}obs},t})  &= \sigma_{{s^*},a}^2\E\left[\frac{1}{n_{a,m}}\right] && n_{a,m} = \sum_{j=1}^{t} 1(A_j=a, M_j = 1)\\
     \V( \mu_{{s^*},a}^{imp,t})  &= \sigma_{{s^*},a}^2\E\left[\frac{1}{n_{a}}\right] && n_{a} = \sum_{j=1}^{t} 1(A_j=a) \\
\end{align*}

Clearly, $n_a \geq n_{a,m}$ and $n_a \geq n_{{s^*},a,o}$ so that

\begin{equation}
    \V( \mu_{{s^*},a}^{{\color{blue}imp},t}) \leq \V( \mu_{{s^*},a}^{{\color{blue}obs},t}) \hspace{2cm}  \V( \mu_{{s^*},a}^{{\color{blue}imp},t}) \leq \V( \mu_{?,a}^{{\color{blue}obs},t})
\end{equation}

implying the multiple imputation approach will have most efficient learning, which could make action \eqref{action-MI} appealing because it could more quickly lead to a correct ordering of actions by rewards. If the missingness rate $p_m$ is very low, then one of $n_{a}\approx n_{s^*,a,o}$ and all the action selection approaches have similar efficiency. If the missingness rate $p_m\approx 1$, $n_{a,m}\approx n_{a}$ and action \eqref{action-M} will have similar efficiency to \eqref{action-MI} because the $?$ mean estimator is used often and its variance is almost as efficient as the MI one. Actions \eqref{action-cons} and \eqref{action-M} may still do badly when occasionally $s^*$ is observed. This is the U-dynamic we describe in the main text where missing-as-state does well for low and high missingness. 

Of course, $p_{s*}=1$ is unrealistic and when this is not the case, there will be some bias. However, we suspect that a \textbf{bias-variance trade-off} with multiple imputation more advantageous when $C$ is more predictive of $S$. Using the missing-as-state approach requires $2|\mathcal{A}|$ estimators instead of $|A|$ estimators for every $c$ and hence can lead to more variable estimators compared to the multiple imputation approach -- if the data are a heap of sand, they are split into more buckets. On the other hand, the multiple imputation option introduces bias, for any states $s$ with rewards that are different from average $|\mu_{s,a}-\mu_a| \gg 0$ and even moreso if those are the same states are often imputed and/or rare ($\hat{n}_{s,a,m}^{K} > n_{s,a,o}$). That said, if a state is very rare, the bias may be less overall impactful (though rare does not per se mean unimportant!). The variance of the MI estimator again depends on the deviations from the (weighted) average reward $\mu_{a}$ as well as on the imputation probabilities but the presence of $\hat{n}_{s,a,m}^K$ in the denominators of each term will contribute to decreasing, particularly for states that are imputed often. As discussed in the Scenarios in Section \ref{scenarios}, the exact trade-off will depend on the costs of choosing a wrong action (i.e., the regret), the equality of the imputations $\hat{p}_{ts}$, and the nature of the missingness.

\subsubsection{Proofs of Means and Variances}

\begin{theorem} \label{mu-obs-bias}  $\E[\mu_{s,a}^{{\color{blue}obs},t}] = \mu_{s,a}$
\end{theorem}

\begin{proof}
    \begin{align}
     \E[\mu_{x,a}^{obs,t}] &= \E[\E[\mu_{x,a}^{obs,t}|S_{1:t}^{obs},A_{1:t},M_{1:t}]] \nonumber \\
     &=  \E\left[\frac{\sum_{j=1}^{t}\E[R_{j}|S_j = x,A_j = a,M_j=0, S_{1:t}^{-j}, A_{1:t}^{-j},M_{1:t}^{-j}] 1(S_{j}=x, A_{j}=a, M_j=0)}{\sum_{j=1}^{t}1(S_{j}=x, A_{j}=a, M_j=0)}\right] \label{lin-con}\\
     &=  \E\left[\frac{\sum_{j=1}^{t}\E[R_{j}(a)|S_j = x,A_j=a, M_j=0, S_{1:t}^{-j},A_{1:t}^{-j},M_{1:t}^{-j}] 1(S_{j}=x, A_{j}=a, M_j=0)}{\sum_{j=1}^{t}1(S_{j}=x, A_{j}=a, M_j=0)}\right] \label{lin-cons}\\
     &=  \E\left[\frac{\sum_{j=1}^{t}\E[R_{j}(a)|S_j = x,A_j=a, A_{1:t}^{-j}] 1(S_{j}=x, A_{j}=a, M_j=0)}{\sum_{j=1}^{t}1(S_{j}=x, A_{j}=a, M_j=0)}\right] \label{lin-unconf1}\\
     &=  \E\left[\frac{\sum_{j=1}^{t}\E[R_j(a)|S_j=x] 1( S_{j}=x, A_{j}=a,  M_j=0)}{\sum_{j=1}^{t}1( S_{j}=x, A_{j}=a, M_j=0)}\right] \label{lin-unconf2}\\
     &=  \mu_{x,a} \E\left[\frac{\sum_{j=1}^{t} 1( S_{j}=x, A_{j}=a,  M_j=0)}{\sum_{j=1}^{t}1( S_{j}=x, A_{j}=a, M_j=0)}\right] \nonumber \\
     &= \mu_{x,a} \nonumber
\end{align}
\noindent Line \ref{lin-con} uses the fact that the indicator means we can condition on $S_j=x,A_j=a,M_j=0$ in the expectation for free. Line \ref{lin-cons} applies SUTVA. Line \ref{lin-unconf1} uses the contextual bandit independence of states to drop $S_{1:t}^{-j}$ and the unconfoundedness and independence (across time) of the missingness to drop $M_j=0$ and $M_{1:t}^{-j}$. Line \ref{lin-unconf2} uses the  unconfoundedness of actions assumption  to drop the action history and current action.

\end{proof}

\begin{theorem}
Let $n_{s,a,o} = \sum_{j=1}^{t}1(S_j = s, A_j=a, M_j=0)$ and $\sigma_{s,a}^2=\V(R(a)|S=s)$. $\V(\mu_{s,a}^{{\color{blue}obs},t}) = \sigma_{s,a}^2\E\left[\frac{1}{n_{s,a,o}}\right]$
\end{theorem}
\begin{proof}
    \begin{align*}
        \V( \mu_{s,a}^{{\color{blue}obs},t}) &= \E[\V( \mu_{s,a}^{{\color{blue}obs},t}|S_{1:t}^{obs},A_{1:t},M_{1:t})] +  \V(\E[ \mu_{s,a}^{{\color{blue}obs},t}|S_{1:t}^{obs},A_{1:t},M_{1:t}]) =  \E[\frac{\sigma_{s,a}^2}{n_{s,a,o}}]  + 0 = \sigma_{s,a}^2\E\left[\frac{1}{n_{s,a,o}}\right] \\
    \end{align*}    
\end{proof}

\begin{theorem}
Let $\mu_a = \sum_{s}\mu_{s,a}p_s$ be the overall average reward (given $C=c$). $\E[\mu_{?,a}^{{\color{blue}obs},t}] = \mu_{a}$.
\end{theorem}

\begin{proof}
\begin{align*}
\E[R_j|A_j=a,A_{1:t}^{-j}, M_j=1, M_{1:t}^{-j}] &= \E[R_j(a)|A_j=a,A_{1:t}^{-j}, M_j=1, M_{1:t}^{-j}]  \\ 
&= \E[\E[R_j(a)|A_j=a,A_{1:t}^{-j}, M_j=1, M_{1:t}^{-j},S_j] \mid A_j=a,A_{1:t}^{-j}, M_j=1, M_{1:t}^{-j} ] \ \\ 
                    &= \E[\E[R_j(a)|S_{j}] \mid A_j=a,A_{1:t}^{-j}, M_j=1, M_{1:t}^{-j}] \\
                    &= \sum_{s} \E[R_j(a)|S_{j}=s]P(S_j=s|A_j=a,A_{1:t}^{-j}, M_j=1, M_{1:t}^{-j}) \\
                    &= \sum_{s} \mu_{s,a}P(S_j=s|M_j=1, M_{1:t}^{-j}) \\
                    &= \sum_{s}\mu_{s,a}p_s  \\ 
&\\ 
    \E[\E[\mu_{?,a}^{{\color{blue}obs},t}|A_{1:t}, M_{1:t}]] &= \E\left[\frac{\sum_{j=1}^{t} \mu_a 1(A_j=a, M_j = 1)}{\sum_{j=1}^{t} 1(A_j=a, M_j = 1)} \right] = \mu_a
\end{align*}

\noindent Here we use the assumption that given $M_j=1$, the true state $S_j=s$ is independent of any actions taken past, present, and future. Then we use the fact that the missingness is MAR to drop those from the conditioning.
\end{proof}

\begin{theorem}
Let $n_{a,m} = \sum_{j=1}^{t}1(A_j=a, M_j=1)$ and $\sigma_{s,a}^2=\V(R(a)|S=s)$. The variance is $ \V(\mu_{?,a}^{{\color{blue}obs},t}) = \sigma_a^2\E\left[\frac{1}{n_{a,m}}\right]$ where $\sigma_a^2 = \V(R_j(a))  = \E[\V(R_j(a)|S_j)] +\V(\E[R_j|S_j]) = \left[\sum_{s}\sigma_{s,a}^2p_{s}+ \sum_{s}\mu_{s,a}^2p_s- \left(\sum_{s}\mu_{s,a}p_{s}\right)^2 \right]$.
\end{theorem}

\noindent Intuitively, for a given $a$, its variance is larger if the variances within $s$ groups are larger on average and if the means of those groups vary more (across-variance), and its variance shrinks as the (expected) number of observations where the state is missing and we take action $a$ increases.

\begin{proof}
Leveraging the same kinds of arguments as in the mean case and using independence of the $R_j(a)$ across time given the history to take the variance within the sum:

\begin{align*}
    \V(\mu_{?,a}^{imp,t}) &= \V(\E(\mu_{?,a}^{imp,t}|A_{1:t},M_{1:t})) + \E(\V(\mu_{?,a}^{imp,t}|A_{1:t},M_{1:t})) \\
    &= 0 + \E[\frac{1}{n_{a,m}^2} \sum_{j=1}^{t}\V(R_j(a)|A_j=a,M_j=1,A_{1:t}^{-j},M_{1:t}^{-j})1(A_j=a,M_j=1)]\\
    &\\ 
    \V(R_j(a)|A_j=a,M_j=1,A_{1:t}^{-j},M_{1:t}^{-j})
    &= \E[\V(R_j(a)|S_j,A_j=a,M_j=1,A_{1:t}^{-j},M_{1:t}^{-j})|A_j=a,M_j=1,A_{1:t}^{-j},M_{1:t}^{-j}] \\
    &\hspace{1cm}+ \V[\E(R_j(a)|S_j,A_j=a,M_j=1,A_{1:t}^{-j},M_{1:t}^{-j})|A_j=a,M_j=1,A_{1:t}^{-j},M_{1:t}^{-j}] \\
    &= \E[\V(R_j(a)|S_j)|A_j=a,M_j=1,A_{1:t}^{-j},M_{1:t}^{-j}] \\
    &\hspace{1cm}+ \V[\E(R_j(a)|S_j)|A_j=a,M_j=1,A_{1:t}^{-j},M_{1:t}^{-j}] \\
     &= \E[\V(R_j(a)|S_j)|M_j=1] + \V[\E(R_j(a)|S_j)|M_j=1] \\
     &= \E[\V(R_j(a)|S_j)] + \V[\E(R_j(a)|S_j)] \\
     &= \V(R_j(a)) \\
    \\
    \\
    \V(\mu_{?,a}^{imp,t})&=  \E\left[\frac{1}{n_{a,m}^2} \sum_{j=1}^{t}\V(R_j(a))1(A_j=a,M_j=1)\right]\\
    &=  \E\left[\frac{1}{n_{a,m}^2} \sum_{j=1}^{t}[\E[\V(R_j(a)|S_j)]+\V(\E[R_j(a)|S_j])]1(A_j=a,M_j=1)\right]\\
    &=  \E\left[\frac{1}{n_{a,m}^2} \sum_{j=1}^{t}[\sum_{s}\sigma_{s,a}^2p_{s}+\V(\mu_{S_j,a})]1(A_j=a,M_j=1)\right]\\
    &=  \E\left[\frac{1}{n_{a,m}} \left[\sum_{s}\sigma_{s,a}^2p_{s}+ \sum_{s}\mu_{s,a}^2p_s - \left(\sum_{s}\mu_{s,a}p_{s}\right)^2 \right]\right]\\
    & \E\left[\frac{1}{n_{a,m}}\right] * \left[\sum_{s}\sigma_{s,a}^2p_{s}+ \sum_{s}\mu_{s,a}^2p_s- \left(\sum_{s}\mu_{s,a}p_{s}\right)^2 \right] 
\end{align*}    
\end{proof}

\textbf{Deriving Results for $\mu_{s,a}^{{\color{blue}imp},t}$ under MAR}

Recall that $\hat{S}_j$ are drawn  at time $j$ with probabilities $\tilde{p}_{1j},...,\tilde{p}_{|S|j}$. These probabilities could be set to $p_1,...,p_{|S|}$ (the true probabilities) or they could be estimates which may be updated after each time step.\footnote{I believe nothing in the proof breaks in the latter case. I found this a little surprising at first but I think it is all about the unconfoundedness, the fact that the imputation has no selection effect on which $R_i's$ you observe given you take action $a$. For the rest, having imputation probabilities depend on the data just makes calculating expected values of the $\hat{n}_{s,a,m}^K$ parts really complicated.}
 In either case, we assume they are drawn at random in the unconfounded way described above.

\begin{theorem}\label{mu-imp-bias}
    Let  $n_{s,a,o} = \sum_{j=1}^{t}1(S_{j}=s,A_{j}=a, M_j=0)$ and $\hat{n}_{s,a,m}^K=\frac{1}{K}\sum_{j=1}^{t}\sum_{k=1}^{K} 1(\hat{S}_j^k=s,A_{j}=a, M_j =1)$, leaving a $t$ index implicit. Also, let $\mu_a = \sum_{s}\mu_{s,a}p_s$. Under the assumptions listed above,  
    
\begin{equation}
     \E[\mu_{s,a}^{{\color{blue}imp},t,K}] =  \mu_{s,a} \E\left[ \frac{n_{s,a,o}}{n_{s,a,o} + \hat{n}_{s,a,m}^K}\right] +  \mu_{a} \E\left[\frac{\hat{n}_{s,a,m}^K}{n_{s,a,o} + \hat{n}_{s,a,m}^K}\right] 
\end{equation}

 \noindent equivalently,
 \begin{equation*}
     \E[\mu_{s,a}^{{\color{blue}imp},t,K}] = \mu_{s,a}\left[\E\left[ \frac{n_{s,a,o} + p_s \hat{n}_{s,a,m}^K}{n_{s,a,o} + \hat{n}_{s,a,m}^K}\right] \right] +  \left( \sum_{s'\neq s}\mu_{s',a}p_{s'}\right) \E\left[\frac{\hat{n}_{s,a,m}^K}{n_{s,a,o} + \hat{n}_{s,a,m}^K}\right]
 \end{equation*}

 \noindent also equivalently, using the fact that for any random variable $Y\in [0,1]$, $\E[cY+d(1-Y)]=(c-d)\E[Y]+d$,

\begin{equation}\label{eq-biasform}
\E[\mu_{s,a}^{{\color{blue}imp},t,K}] = \mu_{s,a}+  (\mu_a - \mu_{s,a})\E\left[ \frac{\hat{n}_{s,a,m}^K}{n_{s,a,o} + \hat{n}_{s,a,m}^K}\right] 
\end{equation}

\noindent which directly gives the expression for the bias $\E[\mu_{s,a}^{{\color{blue}imp},t,K}]-\mu_{s,a}$.
\end{theorem}

\noindent \textbf{Interpretation}:  The first form shows that $\E[\mu_{s,a}^{{\color{blue}imp},t,K}]$ is a convex combination between $\mu_{s,a}$ and the overall mean $\mu_{a}$, which will in general be biased. However, if there exists an $s^*$ such that $p_{s^*}\approx 1$, then, as the second form shows, we get roughly $\mu_{s^*,a}$ -- the bias may not be bad for the highly likely state, though it could be bad for any rare states. In the extreme, if for some state $s^*$, $p_{s^*}=1$, we get unbiasedness for that state and means for any other $s\neq s^*$ are not relevant.

\begin{align*}
    \E[\mu_{s^*,a}^{imp,t,K}]  &= \mu_{s^*,a}\E\left[ \frac{n_{s^*,a,o} +  \hat{n}_{s^*,a,m}^K}{n_{s^*,a,o} + \hat{n}_{s^*,a,m}^K}\right]  = \mu_{s^*,a}
\end{align*}

\noindent The third form (Equation \ref{eq-biasform}) highlights that the bias also depends on how different the overall mean is from the $s$ specific means $(\mu_{a}-\mu_{s,a})$ and on the missingness rate. If the missingness rate is lower, then the expectation involving $\hat{n}_{s,a,m}^K$ in the numerator is closer to $0$.

\begin{proof}
    \noindent Let $H_t =  (S_{1:t}^{obs}, A_{1:t}, M_{1:t}, \hat{S}_{1:t}^K)$.
    
\begin{align}
    &\E[\mu_{s,a}^{{\color{blue}imp},t}] \nonumber \\
    &=\E[\E[\mu_{s,a}^{{\color{blue}imp},t}|H_t]] \nonumber \\ 
    &= \E\left[ \frac{\sum_{j=1}^{t}\E[R_j|H_t]1(S_{j}=s,A_{j}=a, M_j =0) + \frac{1}{K}\sum_{k=1}^{K}\E[R_j|H_t]1(\hat{S}_j^k=s,A_{j}=a, M_j =1)}{\sum_{j=1}^{t}1(S_{j}=s,A_{j}=a, M_j=0) + \frac{1}{K}\sum_{k=1}^{K}1(\hat{S}_j^k=s,A_{j}=a, M_j =1)}\right] \nonumber \\ 
    &= \E\left[ \frac{\sum_{j=1}^{t}\E[R_j(a)|A_t = a, S_t =s, M_j = 0, H_t^{-j}]1(S_{j}=s,A_{j}=a, M_j =0)}{n_{s,a,o} + \hat{n}_{s,a,m}^K}\right]  \nonumber \\ 
    &\hspace{1cm}+ \E\left[\frac{\frac{1}{K}\sum_{j=1}^{t}\sum_{k=1}^{K}\E[R_j(a)|A_t=a,\hat{S}_t^k=s, M_j =1 , H_t^{-j}]1(\hat{S}_j^k=s,A_{j}=a, M_j =1)}{n_{s,a,o} + \hat{n}^K_{s,a,m}}\right] \nonumber \\
     &= (\text{Part I}) + (\text{Part II}) \nonumber  
\end{align}
For Part I, the logic is essentially the same as in the proof of Theorem \ref{mu-obs-bias}, now with the addition of dropping $\hat{S}_{1:t}^K$ because of its unconfoundedness given $S_t$.

\begin{align}
   \text{Part I} &= \E\left[ \frac{\sum_{j=1}^{t}\E[R_j(a)|S_t=s]1(S_{j}=s,A_{j}=a, M_j =0)}{n_{s,a,o} + \hat{n}_{s,a,m}^K}\right] \nonumber \\ 
   &= \E\left[ \frac{\sum_{j=1}^{t}\mu_{s,a}1(S_{j}=s,A_{j}=a, M_j =0)}{n_{s,a,o} + \hat{n}_{s,a,m}^K}\right] \nonumber \\
   &= \mu_{s,a} \E\left[ \frac{n_{s,a,o}}{n_{s,a,o} + \hat{n}_{s,a,m}^K}\right] \\
\end{align}

\noindent For Part II, we have to do an additional Law of Total Expectation Step to condition on $S$ before it is possible to apply unconfoundedness.

\begin{align}
  \E[R_j(a)|A_j=a, \hat{S}_j^k = s, M_j =1 , H_{t}^{-j}] 
  &= \E[\E[R_j(a)|S_j, A_j=a, \hat{S}_j^k = s, M_j =1 , H_{t}^{-j}]|A_j=a, \hat{S}_j^k = s, M_j =1 , H_{t}^{-j}] \nonumber \\
&= \E[\E[R_j(a)|S_j]|A_j=a, \hat{S}_j^k = s, M_j =1 , H_{t}^{-j}]  \nonumber \\
&= \sum_{s}\mu_{s,a}P(S_j=s|A_j=a,\hat{S}_j^k=s,M_j=1,H_t^{-j}) \label{lin-history}\\ 
&= \sum_{s}\mu_{s,a}P(S_j=s|M_j=1) \label{lin-drophistory}\\
&=\sum_{s}\mu_{s,a}P(S_j=s) \label{lin-dropmiss}\\
&=\sum_{s}\mu_{s,a}p_s \nonumber\\
&=\mu_a  \nonumber
\end{align}

\noindent The move from line \ref{lin-history} to line \ref{lin-drophistory} is justified by the fact that if $M_j=1$, then $S_j=s$ has no relationship to the actions taken, past or present, or to the imputation, because it is unobserved. That is, without the $M_j=1$, we could not do this simplification because in general future and present actions can reflect the state $S_t$ and we'd have a more complicated probability on our hands in line \ref{lin-history}. In the next line \ref{lin-dropmiss}, we use the independence between the missingness mechanism and state to drop $M_j=1$ (note: conditioning on a fully observed part of the state $S_1=s_1$ is implicit throughout). Applying the above calculation to the full Part II expression, we have

\begin{align}
    \text{Part II} &=  \E\left[\frac{\frac{1}{K}\sum_{j=1}^{t}\sum_{k=1}^{K}\mu_a 1(\hat{S}_j=s,A_{j}=a, M_j =1)}{n_{s,a,o} + \hat{n}_{s,a,m}^K}\right] =  \mu_{a} \E\left[\frac{\hat{n}_{s,a,m}^K}{n_{s,a,o} + \hat{n}_{s,a,m}^K}\right] \nonumber
\end{align}

\noindent Putting everything together gives:

\begin{align}
     \E[\mu_{s,a}^{{\color{blue}imp},t,K}]     &= \mu_{s,a} \E\left[ \frac{n_{s,a,o}}{n_{s,a,o} + \hat{n}_{s,a,m}^K}\right] +  \left( \sum_{s'}\mu_{s',a}p_{s'}\right) \E\left[\frac{\hat{n}_{s,a,m}^K}{n_{s,a,o} + \hat{n}_{s,a,m}^K}\right] \nonumber \\
    &= \mu_{s,a}\E\left[ \frac{n_{s,a,o} + p_s \hat{n}_{s,a,m}^K}{n_{s,a,o} + \hat{n}_{s,a,m}^K}\right]  +  \left( \sum_{s'\neq s}\mu_{s',a}p_{s'}\right) \E\left[\frac{\hat{n}_{s,a,m}^K}{n_{s,a,o} + \hat{n}_{s,a,m}^K}\right] \label{lin-result-form2}
\end{align}
\end{proof}

\begin{theorem}
Under the assumptions listed above,   and letting $\hat{p}_{js}^K= \frac{1}{K}\sum_{k=1}^{K}1(\hat{S}_j=s)$,

    \begin{align*}
    \V(\mu_{s,a}^{{\color{blue}imp},t}) &=  \sigma_{s,a}^2\E\left[\frac{n_{s,a,o}}{(n_{s,a,o}+\hat{n}_{s,a,m}^K)^2}\right] + \sigma_a^2\E\left[\frac{\sum_{j=1}^{t}(\hat{p}_{js}^K)^2 1(A_{j}=a, M_j =1)}{(n_{s,a,o}+\hat{n}_{s,a,m}^K)^2} \right]  +  \left(\mu_{s,a}-\mu_{a}\right)^2\V\left(\frac{n_{s,a,o}}{n_{s,a,o} + \hat{n}_{s,a,m}^K}\right)\\
\end{align*}

\noindent where $\sigma_a^2 = \V(R_j(a))  = \E[\V(R_j(a)|S_j)] +\V(\E[R_j|S_j]) = \left[\sum_{s}\sigma_{s,a}^2p_{s}+ \sum_{s}\mu_{s,a}^2p_s- \left(\sum_{s}\mu_{s,a}p_{s}\right)^2 \right]$. 
\end{theorem}

\begin{proof}
\noindent  Let $H_t =  (S_{1:t}^{obs}, A_{1:t}, M_{1:t}, \hat{S}_{1:t}^K)$. We start with calculating $\V(\mu_{s,a}^{{\color{blue}imp},t,K}|H_t)$. First, note that conditional on the history, $R_j$ and $ R_i$ for $i\neq j$ are independent so we can bring the variance inside the sum below. We can also separate into a sum of variances within the sum over $j$ because $1(S_j,A_j=a,M_j=0)1(\hat{S}_j^k=s,A_j=a,M_j=1)=0$ always since both cannot be true at once. 

\begin{align*}
  \V(\mu_{s,a}^{{\color{blue}imp},t,K}|H_t) &=  \frac{\sum_{j=1}^{t}\V(R_j|H_t]1(S_{j}=s,A_{j}=a, M_j =0) + \sum_{j=1}^{t}\V(R_j|H_t)  \left(\frac{1}{K}\sum_{k=1}^{K}1(\hat{S}_{j}^k=s)\right)^21(A_{j}=a, M_j =1)}{(\sum_{j=1}^{t}1(S_{j}=s,A_{j}=a, M_j=0) + \frac{1}{K}\sum_{k=1}^{K}1(\hat{S}_j^k=s,A_{j}=a, M_j =1))^2} \nonumber \\ 
\end{align*}

\noindent Again splitting this into two parts, we have that for the left part, the same unconfoundedness logic used in the proof of Theorem \ref{mu-imp-bias} applies:

\begin{align*}
   \text{Part I}
   &= \frac{\sum_{j=1}^{t}\V(R_j(a)|S_t=s)1(S_j=s,A_j=a,M_j=0) }{(n_{s,a,o}+\hat{n}_{s,a,m}^K)^2} = \frac{\sigma_{s,a}^2n_{s,a,o} }{(n_{s,a,o}+\hat{n}_{s,a,m}^K)^2}  \\ 
\end{align*}

\noindent For the right part, we apply the Law of Total Variance and use a similar logic to in Theorem  \ref{mu-imp-bias} to reason:

\begin{align*}
    \V(R_j(a)|A_j=a,M_j=1,H_t) &= \E[\V(R_j(a)|S_t, A_j=a,M_j=1,H_t)|A_j=a,M_j=1,H_t] \\
    &\hspace{1cm}+ \V[\E(R_j(a)|S_t, A_j=a,M_j=1,H_t)|A_j=a,M_j=1,H_t]\\ 
     &= \E[\V(R_j(a)|S_t)|A_j=a,M_j=1,H_t] + \V[\E(R_j(a)|S_t)|A_j=a,M_j=1,H_t]\\ 
     &= \sum_{s}\sigma^2_{s,a}P(S_t=s|A_j=a,M_j=1,H_t) + \V[\mu_{S_{t},a}|A_j=a,M_j=1,H_t]\\ 
     &= \sum_{s}\sigma^2_{s,a}p_s + \sum_{s}\mu_{s,a}^2p_s -\left(\sum_{s}\mu_{s,a}p_{s}\right)^2 \\ 
     &= \sigma_a^2
\end{align*}

Bringing this into the full part II term:

\begin{align*}
   \text{Part II}
   &= \frac{\sum_{j=1}^{t} \V(R_j(a)|A_j=a,M_j=1,H_t)\left(\frac{1}{K}\sum_{k=1}^{K}1(\hat{S}_{j}^k=s)\right)^21(A_{j}=a, M_j =1) }{(n_{s,a,o}+\hat{n}_{s,a,m}^K)^2} \\
   &= \frac{\sum_{j=1}^{t}\sigma_a^2 
   (\hat{p}_{js}^K)^2 1(A_{j}=a, M_j =1) }{(n_{s,a,o}+\hat{n}_{s,a,m}^K)^2} \\
   &= \sigma_a^2\frac{\sum_{j=1}^{t}(\hat{p}_{js}^K)^2 1(A_{j}=a, M_j =1)}{(n_{s,a,o}+\hat{n}_{s,a,m}^K)^2}  \\ 
\end{align*}

\noindent where note that if $K=1$, this term is 

\[ \sigma_a^2\frac{\hat{n}_{s,a,m}^K}{(n_{s,a,o}+\hat{n}_{s,a,m}^K)^2}  \]

\noindent Overall, Parts I and II give:

\begin{align*}
     \E[\V(\mu_{s,a}^{{\color{blue}imp},t}|A_{1:t},M_{1:t},S_{1:t}^{obs},\hat{S}_{1:t})] &= \sigma_{s,a}^2\E\left[\frac{n_{s,a,o}}{(n_{s,a,o}+\hat{n}_{s,a,m}^K)^2}\right] + \sigma_a^2\E\left[\frac{\sum_{j=1}^{t}(\hat{p}_{js}^K)^2 1(A_{j}=a, M_j =1)}{(n_{s,a,o}+\hat{n}_{s,a,m}^K)^2} \right]
\end{align*}

\noindent The full variance is then obtained by applying the Law of Total Variance and plugging in calculations from the proof of Theorem \ref{mu-imp-bias}.

\begin{align*}
    \V(\mu_{s,a}^{{\color{blue}imp},t}) &= \E[\V(\mu_{s,a}^{{\color{blue}imp},t}|A_{1:t},M_{1:t},S_{1:t}^{obs},\hat{S}_{1:t})] +  \V[\E(\mu_{s,a}^{{\color{blue}imp},t}|A_{1:t},M_{1:t},S_{1:t}^{obs},\hat{S}_{1:t})] \\ 
\end{align*}

\noindent For the second term, Note the general result that if $Y\in [0,1]$, then $\V(cY+d(1-Y)) = \V((c-d)Y)=(c-d)^2\V(Y)$. Hence, we have

\begin{align*}
    \V\left[ \mu_{s,a}\frac{n_{s,a,o}}{n_{s,a,o} + \hat{n}_{s,a,m}^K}+ \mu_{a}\frac{\hat{n}_{s,a,m}^K}{n_{s,a,o} + \hat{n}_{s,a,m}^K}\right] &= \left(\mu_{s,a}-\mu_{a}\right)^2\V\left(\frac{n_{s,a,o}}{n_{s,a,o} + \hat{n}_{s,a,m}^K}\right)
\end{align*}

\end{proof}

\newpage 
\subsection{Justifying the Fractional $Q$-Learning Update}\label{app-frac-Qlearn}

We might wonder whether, given a fully observed sequence of $(S_t,A_t,R_t,S_{t+1})$, the sequential fractional Q-learning update recovers the usual Q-learning update if there was no missingness. The answer is that it does not exactly but comes very close.

\begin{theorem}
Suppose $S_t=s$ and $S_{t+1}=s'$ are fully observed, with action $A_t=a$ and reward $R_t=r$. If $s\neq s'$, then 

\begin{equation*}
    Q^{(k)}_{t+1}(s,a) =  Q_t^{(K)}(s,a)+\left(kr + \sum_{j=2}^{k} (-1)^{j-1}c_{j,k} r^j \right)TD_{\gamma}  \ \ \ \ \text{for} \ r=\frac{\alpha}{K}
\end{equation*}

where if $k=1$, the sum is defined as $0$,  where $TD_\gamma=R+\gamma \max_{a'}Q_t^{(K)}(s',a') - Q_t^{(K)}(s,a)$, and where $c_{k,k}=1$, $c_{2,k}=3$ for all $k$ and $c_{j,k+1}=c_{j,k} + c_{j-1,k}$ for $k=2,...,K$.
\end{theorem}

Hence, for $k=K$ and fully observed data, the full update is the usual complete data Q learning update for $(s,a,r',s')$ with learning rate $\alpha$, plus an additional polynomial with alternating signs and all terms smaller than $(\alpha/K)^2$  in magnitude as long as $\alpha/K < 1$.  As $K$ grows, $\alpha/K < 1$ holds eventually and the $r^j$ decrease exponentially, but the $c_{j,k}$ also grow exponentially because of their recursive Fibonacci-sequence-like form (the Fibonacci sequence grows exponentially by Binet's formula).  However, the alternating signs create some cancellation, and simulations indicate that for $\alpha < 1$ and all reasonablly small $K$,

\begin{equation*}
 Q^{(K)}_{t+1}(s,a) \approx   Q_t^{(K)}(s,a)+ \alpha TD_{\gamma}
\end{equation*}

Figure \ref{fig-learnratekadd} displays the size of the learning rate `add-on' term for $k=K$ and $\alpha=1$ over a range of plausible $K$ and shows that even for small $K$, this is orders of magnitude smaller than $\alpha$.

\begin{figure*}[!t]
    \centering
    \includegraphics[width = .6\linewidth]{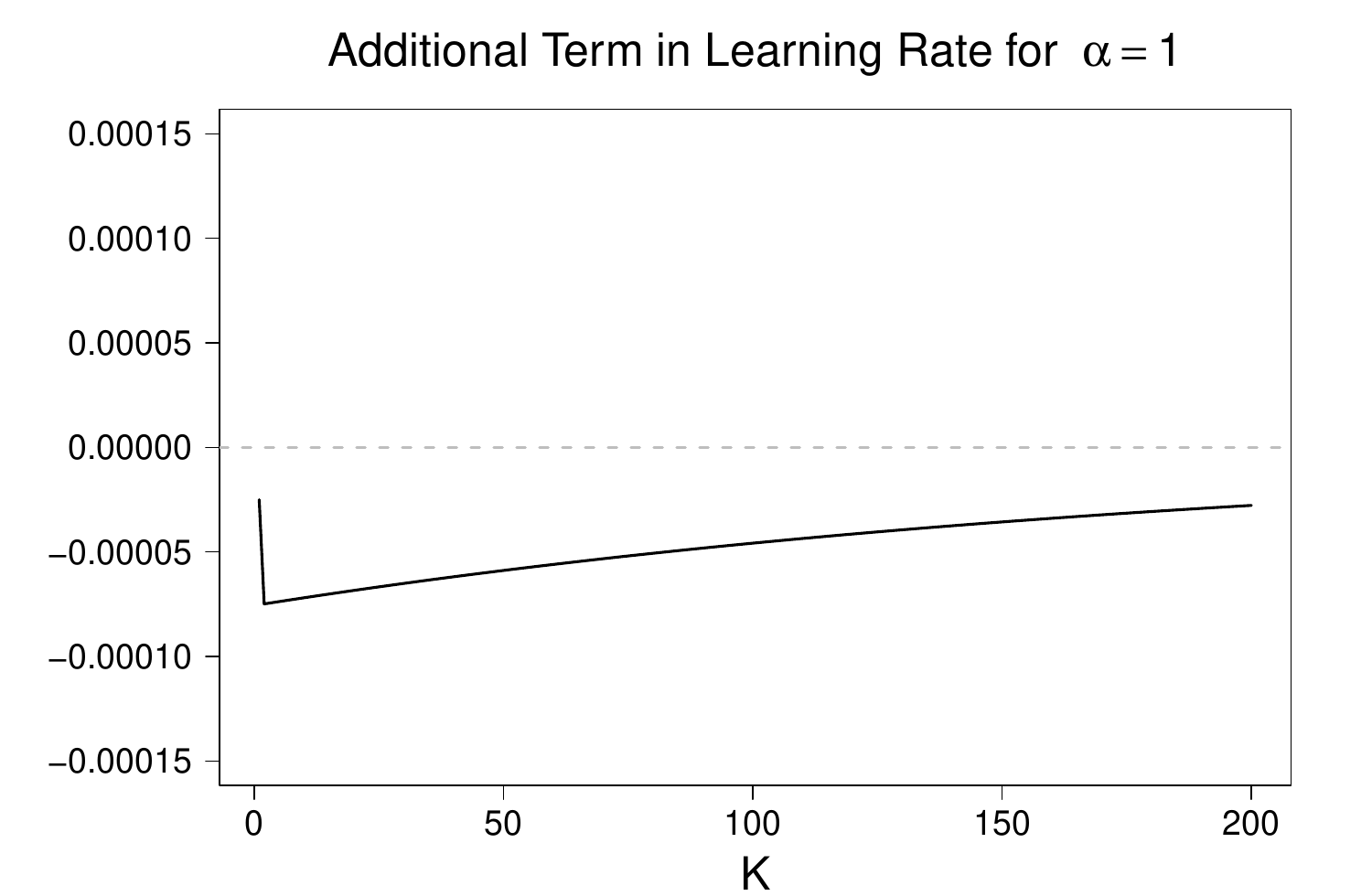}
    \caption{Magnitude of $\sum_{j=2}^{K} (-1)^{j-1}c_{j,K} r^j$ over range of $K$ with $r=1/K$.}
    \label{fig-learnratekadd}
\end{figure*}

If $s=s'$, then this result still applies if $\max_{a'}Q_{t+1}^{(k)}(s,a') = \max_{a'}Q_t^{(K)}(s,a')$ for all $k$, which is more likely as the Q-Learning converges so that small updates do not change the maximum. %

\begin{figure*}
\begin{proof}
If $s\neq s'$, $\max_{a'}Q_{t+1}^{(k)}(s',a') = \max_{a'}Q_t^{(K)}(s',a')$ for all $k$ because $Q_{t+1}^{(k)}(s',a)$ is never updated for any $a$ relative to $Q_{t}^{(K)}$. Let $M=R+\gamma \max_{a'}Q_t^{(K)}(s',a)$ and $L=(M - Q_t^{(K)}(s,a))$ and adopt the convention that a sum from $j=a..b$ is $0$ if $b<a$. Also, let $\frac{\alpha}{K}=r$.
\begin{align*}
    Q_{t+1}^{(1)}(s,a)&= Q_t^{(K)}(s,a) + r(M - Q_t^{(K)}(s,a))\\ 
    Q_{t+1}^{(2)}(s,a)&= {\color{blue}Q_{t+1}^{(1)}(s,a)} + r(M - Q_t^{(1)}(s,a))\\ 
    &={\color{blue} Q_t^{(K)}(s,a) + r\left(M - Q_t^{(K)}(s,a)\right)}  + r\left(M - Q_t^{(K)}(s,a) - r(M - Q_t^{(K)}(s,a))\right)\\ 
    &= Q_t^{(K)}(s,a) + 2r\left(M - Q_t^{(K)}(s,a)\right)  - r^2L\\
     &= \color{orange} Q_t^{(K)}(s,a) + 2rL  - r^2L\\
 Q_{t+1}^{(3)}(s,a)  
 &= {\color{blue}Q_{t+1}^{(2)}(s,a)} + r(M - Q_t^{(2)}(s,a))\\ 
 &= {\color{blue} Q_t^{(K)}(s,a) + 2rL  - r^2L} + r\left(M - Q_t^{(K)}(s,a) - 2rL-r^2L \right)\\ 
 &= {\color{blue} Q_t^{(K)}(s,a) + 3rL  - r^2L}  - 2r^2L + r^3L \\ 
 &= \color{orange} Q_t^{(K)}(s,a) + 3rL   - 3r^2L + r^3L \\ %
Q_{t+1}^{(k)}(s,a)&= Q_t^{(K)}(s,a)+krL + L\left(\sum_{j=2}^{k} (-1)^{j-1}c_{j,k} r^j \right) \\ 
Q_{t+1}^{(k+1)}(s,a)&= Q_{t+1}^{(k)}(s,a) + r(M - Q_{t+1}^{(k)}(s,a))\\  
&= Q_t^{(K)}(s,a)+krL + L\left(\sum_{j=2}^{k} (-1)^{j-1}c_{j,k} r^j \right)  + r\left(L - L\left(\sum_{j=2}^{k} (-1)^{j-1}c_{j,k} r^j \right) \right)\\
&= Q_t^{(K)}(s,a)+(k+1)rL + L\left(\sum_{j=2}^{k} (-1)^{j-1}c_{j,k} r^j \right)  + r\left(- L\left(\sum_{j=2}^{k} (-1)^{j-1}c_{j,k} r^j \right) \right)\\
&= Q_t^{(K)}(s,a)+(k+1)rL + L\left( (-1)^{k}r^{k+1} + \sum_{j=2}^{k} (-1)^{j-1}(c_{j,k} + c_{j-1,k})r^j \right) \\ %
&=  Q_t^{(K)}(s,a)+(k+1)rL + L\left(\sum_{j=2}^{k+1} (-1)^{j-1}c_{j,k+1}r^j \right) \\
\end{align*}

\end{proof}
\end{figure*}

\newpage 
\section{Figure Appendix}
\subsection{Additional MCAR Figures} \label{appendix-mcar}

\noindent%

\begin{figure}[H]
\includegraphics[keepaspectratio=true,scale=0.8]{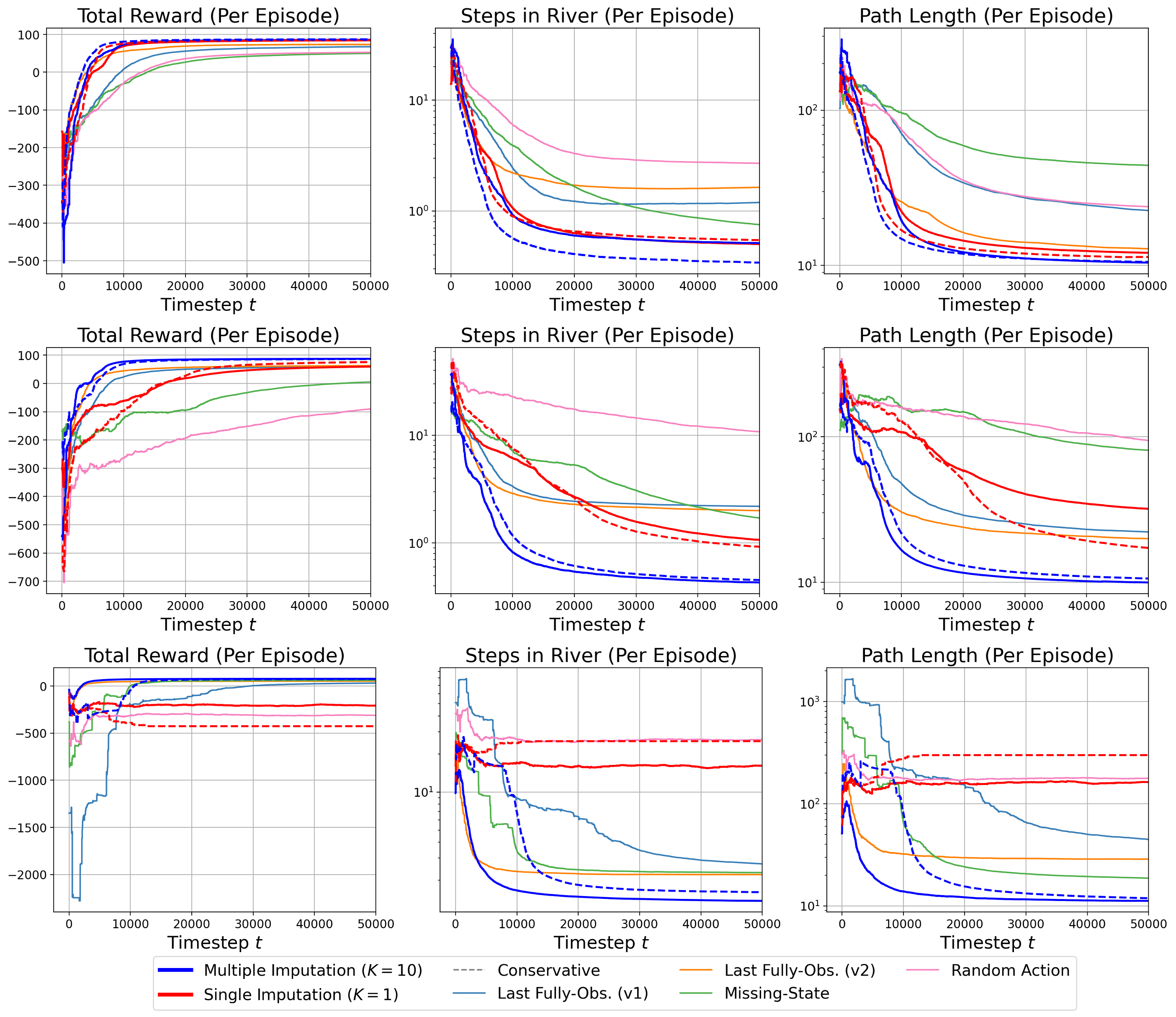}
\caption{\footnotesize{Comparing performance over time for MCAR with $\theta=0.2$ (\textbf{top}), $\theta = 0.4$ (\textbf{middle}), $\theta = 0.8$ (\textbf{bottom}).  The y-axis in each case is the cumulative mean over episodes at time $t$ of the given metric. The environment is set to $P(\text{wind}) = 0.1$, $P(\text{flood\  transition}) = 0.1$. Each method shown for its best $\epsilon,\alpha,\gamma$, and action space option (stay in place allowed or not). Lines represent an average over 5 trials. Note the log scale for the center and right plot. }}\label{fig-extramcar-additional}
\end{figure}

\subsection{Additional MCOLOR Figures}\label{app-mcolor}

\begin{figure}[H]
 \centering
\includegraphics[keepaspectratio=true,scale=1.0]{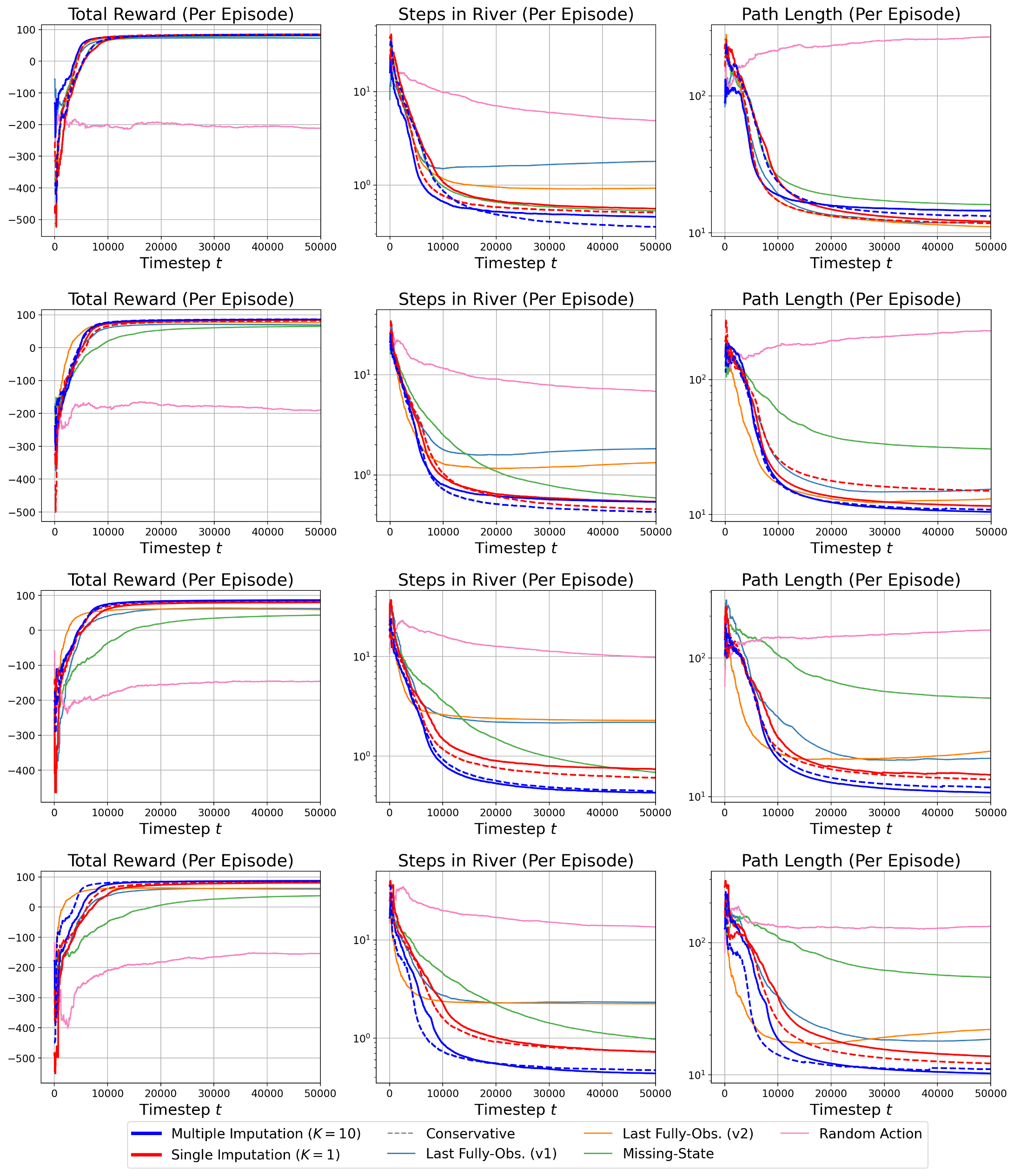}
\caption{\footnotesize{Comparing performance over time for MCOLOR, where missingness is color-specific with rate $\theta_c$ for color $c\in \{g,o,r\}$ (green, orange, red). \textbf{Top}: $(\theta_g,\theta_o,\theta_r) = (0.1,0.2,0.3)$ but only for  $(x,y)$, color never missing. \textbf{Second}: same as previous only now color also missing. \textbf{Third}: $(\theta_g,\theta_o,\theta_r) = (0.2,0.4,0.6)$ but only for  $(x,y)$, color never missing. \textbf{Bottom}: same as previous only now color also missing. This plot appears in Figure \ref{fig-ColorFog} in the main text. The y-axis in each case is the cumulative mean over episodes at time $t$ of the given metric. The environment is set to $P(\text{wind}) = 0.1$, $P(\text{flood\  transition}) = 0.1$. Each method shown for its best $\epsilon,\alpha,\gamma$, and action space option (stay in place allowed or not). Lines represent an average over 5 trials. Note the log scale for the center and right plot. }}\label{fig-extramcolor}%
\end{figure}

\subsection{Additional MFOG Figures}\label{app-mfog}

\begin{figure}
    \centering
\includegraphics[keepaspectratio=true,scale=0.8]{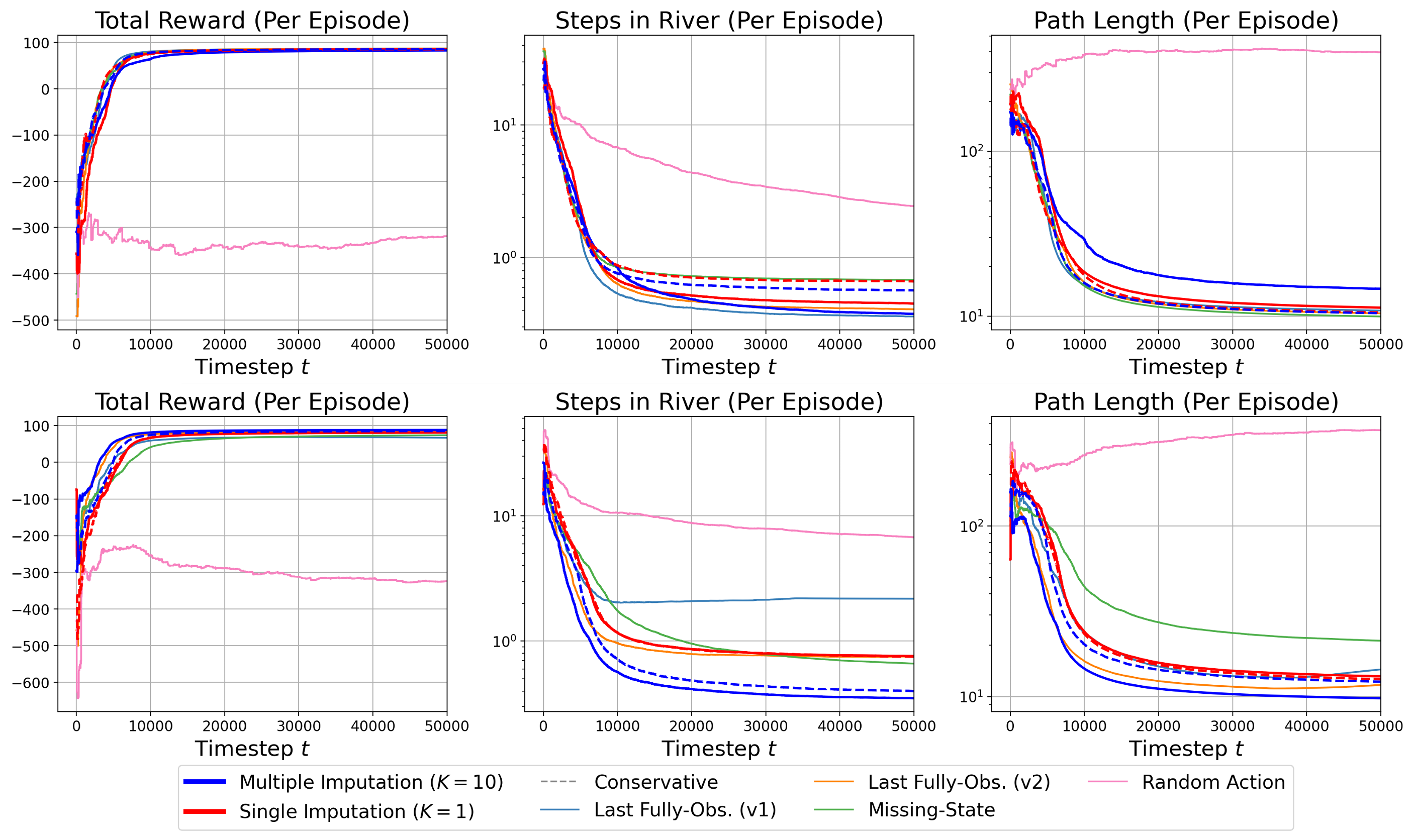}
\caption{\footnotesize{Comparing performance over time for MFOG. \textbf{Top:} Missingness rate inside fog region is $0.5$ while outside is $0$. This plot appears in Figure \ref{fig-ColorFog} in the main text. \textbf{Bottom:} Missingness rate inside fog region is $0.25$ while outside is $0.1$. The y-axis in each case is the cumulative mean over episodes at time $t$ of the given metric. The environment is set to $P(\text{wind}) = 0.1$, $P(\text{flood\  transition}) = 0.1$. Each method shown for its best $\epsilon,\alpha,\gamma$, and action space option (stay in place allowed or not). Lines represent an average over 5 trials. Note the log scale for the center and right plot. }}\label{fig-extramcolor+extramcar}
\end{figure}

\subsection{Choosing $K$} \label{app-choose K}

\begin{minipage}{\textwidth}%
\makebox[\textwidth]{%
  \includegraphics[keepaspectratio=true,scale=1.3]{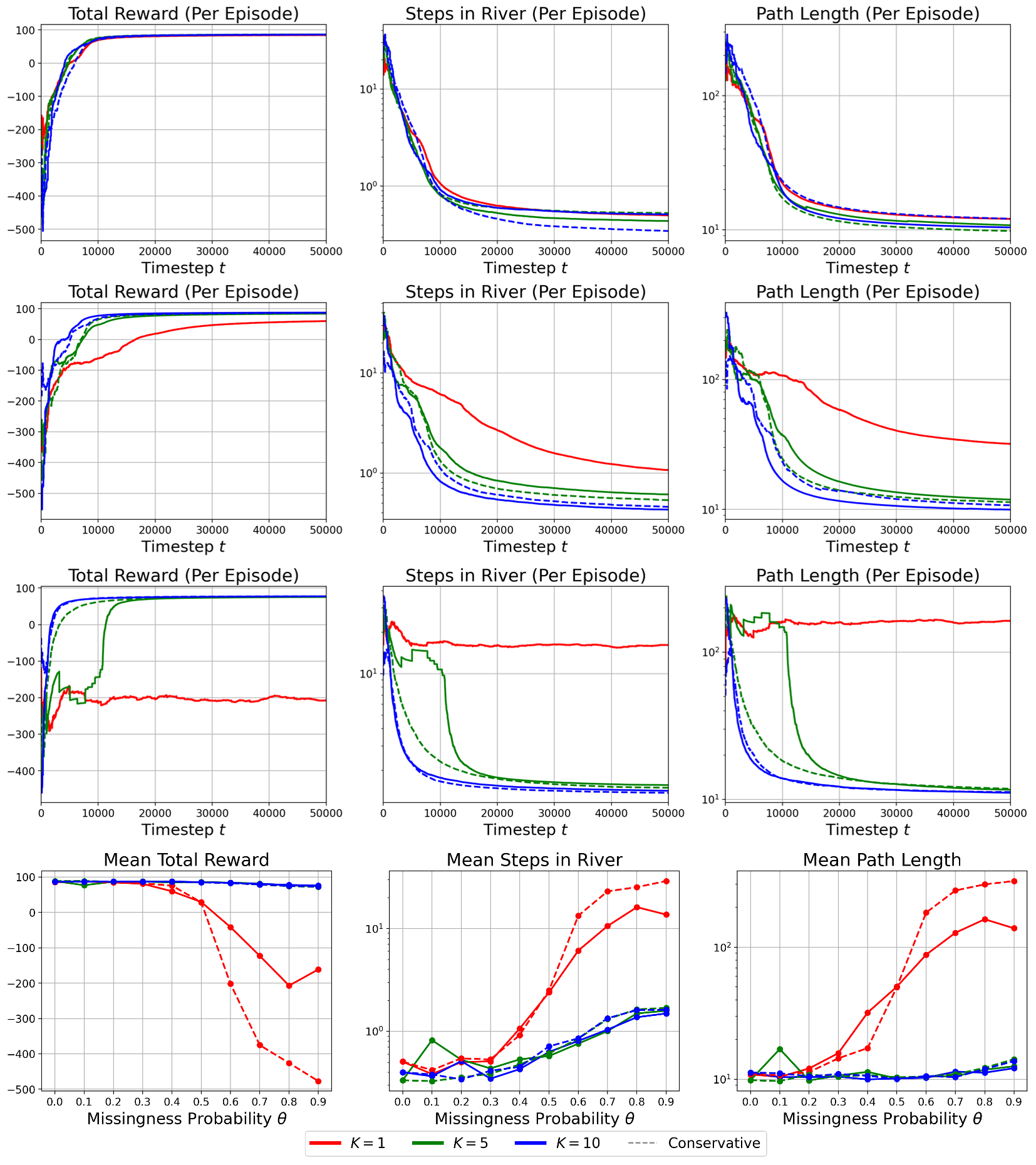}}
\captionof{figure}{\footnotesize{Comparing performance over time for $K=1,5,10$ in the MCAR setting with $\theta = 0.2$ \textbf{(Top)}, $\theta = 0.4$ \textbf{(Second)}, $\theta = 0.8$ \textbf{(Third)}. The y-axis in each case for these first three rows is the cumulative mean over episodes at time $t$ of the given metric. \textbf{Bottom} plot shows how mean performance over 50000 timesteps changes with MCAR missing rate. The environment is set to $P(\text{wind}) = 0.1$, $P(\text{flood\  transition}) = 0.1$. Each method shown for its best $\epsilon,\alpha,\gamma$, and action space option (stay in place allowed or not). Lines represent an average over 5 trials. Note the log scale for the center and right plot. }}\label{fig-extramcolor+compareK}%
\end{minipage}

\subsection{Mixing} \label{app-mixing}

\begin{minipage}{\textwidth}%
\makebox[\textwidth]{%
  \includegraphics[keepaspectratio=true,scale=0.8]{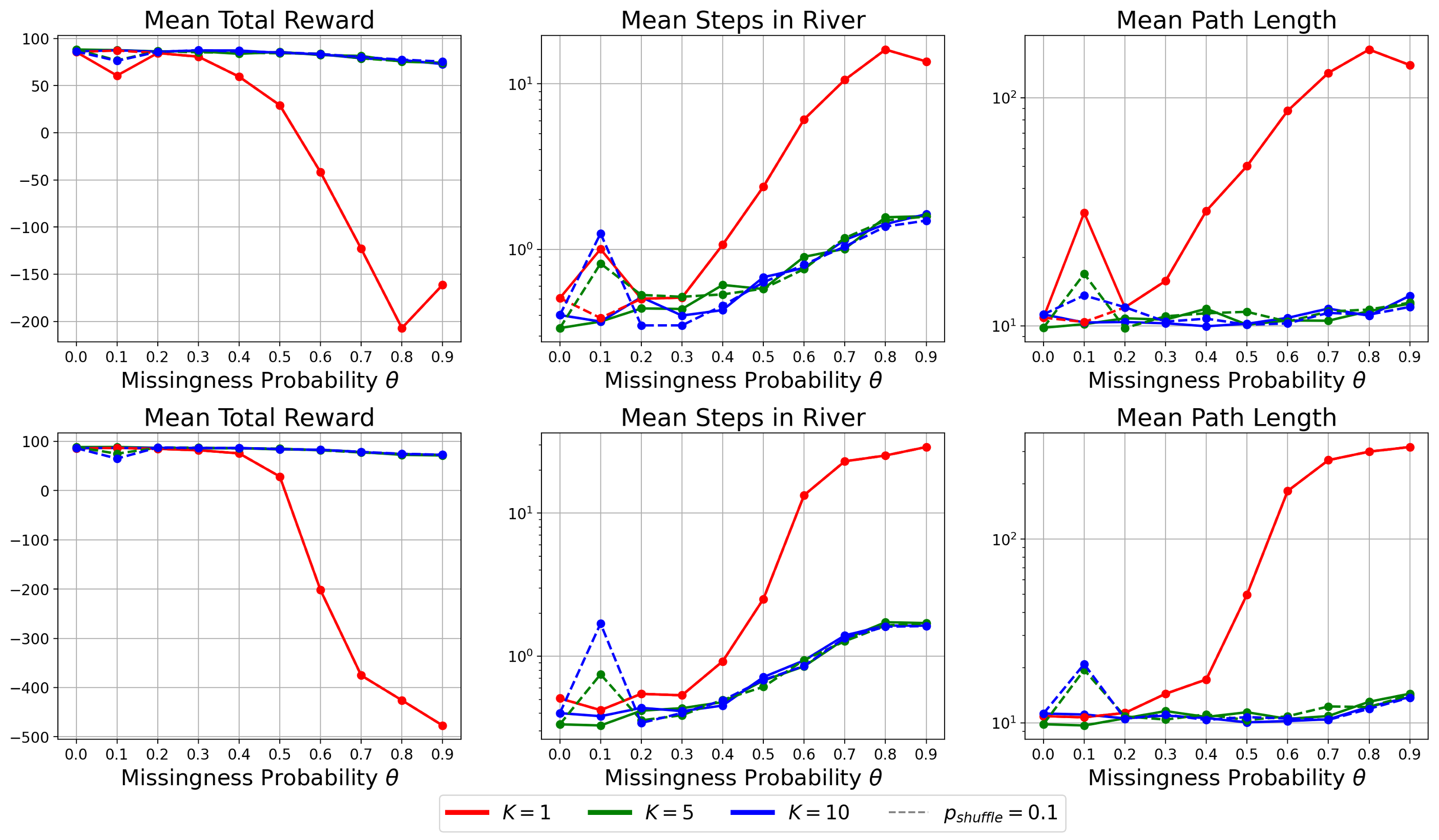}}
\captionof{figure}{\footnotesize{ \textbf{Top}: comparing $K=1,5,10$ with mixing ($p_{\text{shuffle}} = 0.1$, dashed line) and without missing ($p_{\text{shuffle}} = 0$, solid line) for MI with synthetic T-updates. \textbf{Bottom}: The same for MI with conservative T-updates.
The y-axis in each case is the cumulative mean over episodes at time $t$ of the given metric. The environment is set to $P(\text{wind}) = 0.1$, $P(\text{flood\  transition}) = 0.1$. Each method shown for its best $\epsilon,\alpha,\gamma$, and action space option (stay in place allowed or not). Lines represent an average over 5 trials. Note the log scale for the center and right plot. }}\label{fig-extramcar}%
\end{minipage}

\section{Stay In Place} \label{app-stay-in-place}

\begin{minipage}{\textwidth}%
\makebox[\textwidth]{%
  \includegraphics[keepaspectratio=true,scale=0.75]{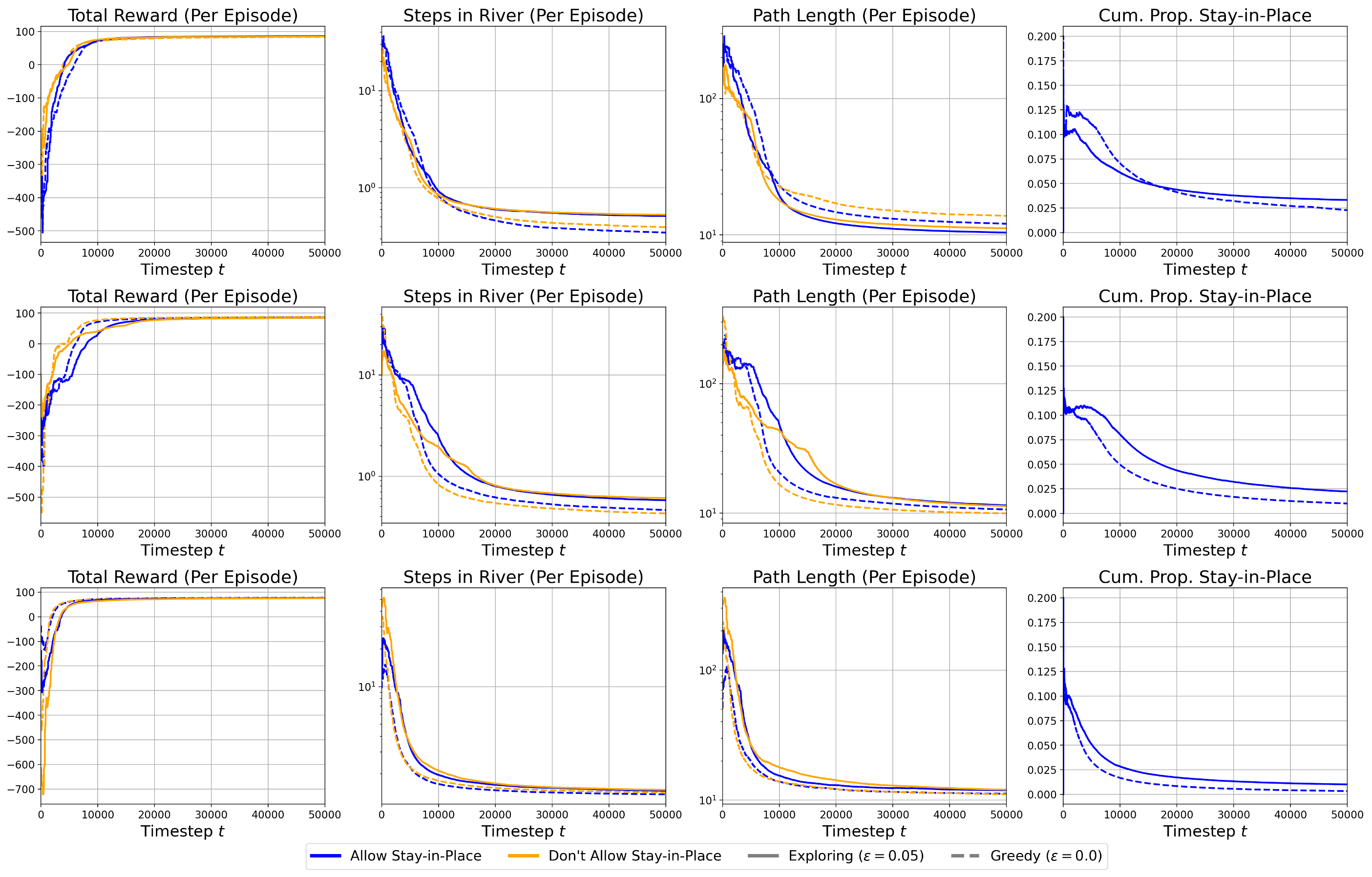}}
\captionof{figure}{\footnotesize{Comparing the role of allowing stay-in-place as an action versus not allowing stay-in-place as an action, as well as investigating the effects of allowing $\epsilon$-greedy exploration ($\epsilon = 0.05$) versus going full greedy ($\epsilon = 0.0$). Performance over time curves are shown for MCAR with $\theta=0.2$ (\textbf{top}), $\theta = 0.4$ (\textbf{middle}), $\theta = 0.8$ (\textbf{bottom}). We use the MI with synthetic T-update with $K=10$ for all curves shown. For the first three columns, the y-axes are the cumulative means per episodes at time $t$ of the given metric. The last column shows the cumulative proportion of timesteps that the agent chose to stay in place. The environment is set to $P(\text{wind}) = 0.1$, $P(\text{flood\  transition}) = 0.1$. Each method shown for its best $\alpha,\gamma$. Lines represent averages over 5 trials. Note the log scale for the two center plots. }}\label{fig-stay-in-place}
\end{minipage}

\end{document}